\documentclass[twoside]{article}

\usepackage[preprint]{aistats2027}
\AtBeginDocument{
    \setlength{\footskip}{20pt}
    \fancyfoot[C]{\thepage}
}
\usepackage[round]{natbib}

\usepackage{amsmath,amsfonts,bm,amssymb,amsthm}

\def\eqref#1{equation~\ref{#1}}
\def\Eqref#1{Equation~\ref{#1}}

\def\1{\bm{1}}

\DeclareMathAlphabet{\mathsfit}{\encodingdefault}{\sfdefault}{m}{sl}
\SetMathAlphabet{\mathsfit}{bold}{\encodingdefault}{\sfdefault}{bx}{n}

\usepackage{thm-restate}

\usepackage{graphicx}
\usepackage{xcolor}
\usepackage{booktabs}
\usepackage[colorlinks=true,linkcolor=blue!60!black,citecolor=blue!60!black,urlcolor=blue!60!black]{hyperref} 

\definecolor{Curve}{rgb}{0.20,0.58,0.82}
\definecolor{Origin}{rgb}{0.72,0.16,0.46}
\definecolor{Tangent}{rgb}{0.76,0.12,0.16}
\definecolor{Normal}{rgb}{0.90,0.55,0.05}
\definecolor{Biased}{rgb}{0.47,0.27,0.72}
\definecolor{Voronoi}{rgb}{0.05,0.50,0.30}

\usepackage{url}
\usepackage{cleveref}
\usepackage{subcaption}

\newcommand{\RR}{\mathbb{R}}
\newcommand{\CC}{\mathbb{C}}
\newcommand{\PP}{\mathbb{P}}
\newcommand{\cF}{\mathcal{F}}

\newcommand{\cH}{\mathcal{H}}
\newcommand{\cC}{\mathcal{C}}
\newcommand{\cX}{\mathcal{X}}
\newcommand{\cP}{\mathcal{P}}
\newcommand{\bw}{\mathbf{w}}
\newcommand{\bA}{\mathbf{A}}

\newcommand{\rank}{\operatorname{rank}}

\newcommand{\im}{\operatorname{im}}

\newcommand{\linspan}{\operatorname{span}}

\newcommand{\Vorloc}{\operatorname{Vor}^{\operatorname{loc}}}
\newcommand{\Vor}{\operatorname{Vor}}
\newcommand{\id}{\operatorname{id}}
\DeclareMathOperator{\sym}{sym}
\DeclareMathOperator{\Res}{Res}

\newtheorem{theorem}{Theorem}[section]
\newtheorem{proposition}[theorem]{Proposition}
\newtheorem{lemma}[theorem]{Lemma}
\newtheorem{corollary}[theorem]{Corollary}
\newtheorem{conjecture}[theorem]{Conjecture}

\theoremstyle{definition}
\newtheorem{definition}[theorem]{Definition}
\newtheorem{remark}[theorem]{Remark}
\newtheorem{example}[theorem]{Example}

\AddToHook{env/theorem/begin}
  {\crefalias{theorem}{theorem}}
\AddToHook{env/proposition/begin}
  {\crefalias{theorem}{proposition}}
\AddToHook{env/lemma/begin}
  {\crefalias{theorem}{lemma}}
\AddToHook{env/corollary/begin}
  {\crefalias{theorem}{corollary}}
\AddToHook{env/conjecture/begin}
  {\crefalias{theorem}{conjecture}}
\AddToHook{env/definition/begin}
  {\crefalias{theorem}{definition}}
\AddToHook{env/example/begin}
  {\crefalias{theorem}{example}}
\AddToHook{env/remark/begin}
  {\crefalias{theorem}{remark}}

\crefname{theorem}{theorem}{theorems}
\Crefname{theorem}{Theorem}{Theorems}
\crefname{lemma}{lemma}{lemmas}
\Crefname{lemma}{Lemma}{Lemmas}
\crefname{proposition}{proposition}{propositions}
\Crefname{proposition}{Proposition}{Propositions}
\crefname{corollary}{corollary}{corollaries}
\Crefname{corollary}{Corollary}{Corollaries}
\crefname{conjecture}{conjecture}{conjectures}
\Crefname{conjecture}{Conjecture}{Conjectures}
\crefname{remark}{remark}{remarks}
\Crefname{remark}{Remark}{Remarks}
\crefname{definition}{definition}{definitions}
\Crefname{definition}{Definition}{Definitions}
\crefname{example}{example}{examples}
\Crefname{example}{Example}{Examples}

\begin{document}
\thispagestyle{plain} 

%

%

\twocolumn[

\aistatstitle{Do Neural Models Prefer Their Submodels?}

\aistatsauthor{\'Alvaro Ribot \And Yulia Alexandr \And  Anna Seigal }

\aistatsaddress{  Harvard University \And  University of Melbourne \And Harvard University } ]

\begin{abstract}
A neural model is the set of functions that can be expressed by a given neural network architecture.
Neural models contain natural submodels, such as those obtained by setting some neurons to be inactive. We study the geometry of neural models and ask whether they prefer their submodels. By this we mean  
whether local minimizers of the loss function lie in a given submodel. 
We introduce a geometric framework to determine if this occurs.
For shallow multilayer perceptrons, we show that subnetworks with inactive hidden neurons are never preferred.  For deep multilayer perceptrons, we give conditions under which subnetworks with inactive hidden neurons are not preferred. For lightning self-attention, lower-rank submodels are usually not preferred, but we identify examples where they are because they lie on a boundary of the model.
Our results show that, aside from a few exceptions, preference for submodels cannot be explained by the geometry of the neural model, and must instead come from other sources such as training dynamics.
\end{abstract}

\section{Introduction}

\begin{figure}
    \centering
    \includegraphics[width=0.76\linewidth]{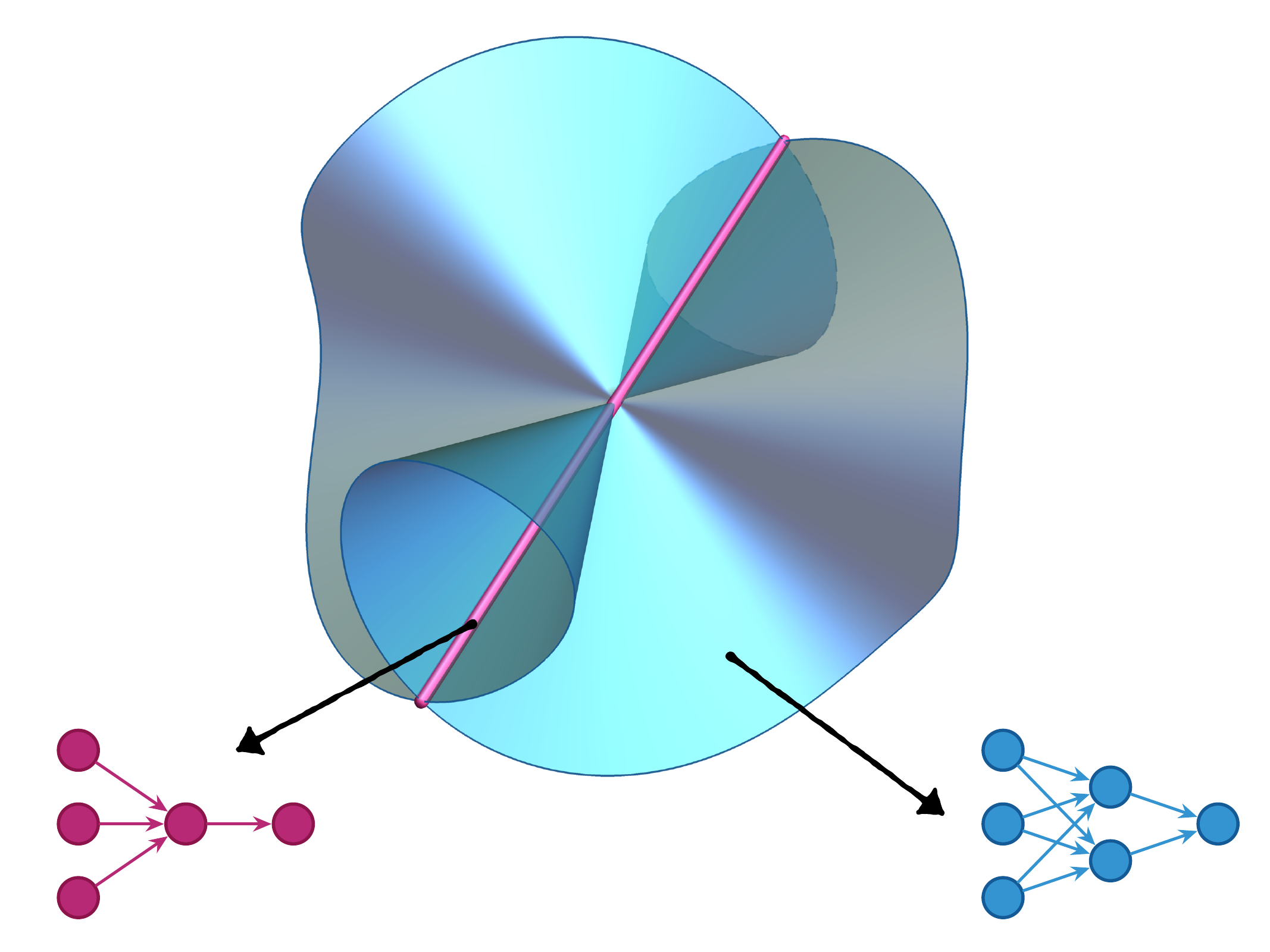}    
    \caption{
    \textbf{No proximity bias toward subnetworks of shallow multilayer perceptrons (MLPs)}.}
    \label{fig:mlp}
\end{figure}

A neural model $\cF$ is the set of functions that can be expressed by a given neural network architecture. Given a target function $h$, we seek 
a minimizer of the distance function from $h$ to $\cF$.
Submodels are subsets of $\cF$ obtained by reducing hyperparameters, such as rank, or by requiring some neurons to be inactive. 
We study whether neural models prefer their submodels or, more specifically, whether a local minimizer of the distance from $h$ to $\cF$ lies in a given submodel $\cF'$. We call this behavior \emph{proximity bias}.

Training a model is known to produce
solutions that lie on, or close to, one of these submodels \citep{frankle2018the,arora2019implicit,shahverdi2026learningrazorsedgeidentifiability}.
This could be due to structure in the data, the optimization dynamics, loss function, or the geometry of the neural model.
Identifying why it occurs is important to understand how models generalize and how to change their behavior.
We seek to isolate the role of the neural model geometry, which affects the dynamics and loss in a downstream~way. 

Compared to general points in a given model, points on submodels may possess more tangent directions to the model or fewer.
Proximity bias can occur only in the latter case, since a minimizer must lie in the polar of the tangent cone.
We show that proximity bias is governed by the local geometry of the model along its submodel. 
Figure~\ref{fig:mlp} illustrates our result for shallow MLPs: the submodel obtained using inactive hidden neurons has more tangent directions than a general point, so proximity bias does not occur.

Low-rank matrix approximation sets the scene. 
The best rank-at-most-$r$ approximation of a matrix $M$ in (e.g.)~Frobenius norm has rank exactly~$r$, unless $\rank(M) < r$, due to \citet{eckart1936approximation}.
That is, a low-rank matrix model uses available degrees of freedom rather than collapsing to an even-lower-rank submodel. The analogous statement holds for symmetric tensors \citep{seigal2017real}. This implies no proximity bias for single-output shallow MLPs with monomial activations. We extend this 
to deep MLPs and those with analytic activation functions, and to lightning self-attention.

\subsection{Summary of contributions}

We study two classes of neural models (also known as neuromanifolds): those corresponding to multilayer perceptrons and to lightning self-attention. For shallow MLPs, we show that the narrower submodel obtained by removing a hidden neuron is not preferred. For deep MLPs, inactive neurons are again not preferred in several cases: when one neuron is inactive in every hidden layer, when the earlier layers are wide enough, when the activations are polynomial of sufficiently high degree, and when the activations are analytic but not polynomial.
For attention, many lower-rank and singular strata do not exhibit proximity bias, but certain exceptional singular or boundary strata do. We show that singularity alone does not determine whether a submodel is preferred; it is governed by a finer geometry of the neural model; see \Cref{fig:implicit-bias-examples} for 
examples.

\subsection{Related work}

\paragraph{Implicit and sparsity bias.}
Trained models have been shown to prefer structured solutions, such as max-margin, low-rank or sparse ones, without them appearing in the objective function. This phenomenon is called \emph{implicit bias}
\citep{soudry2018implicit,gunasekar2017implicit,gunasekar2018characterizing,chizat2020implicit}. A preference for networks with fewer active neurons is called \emph{sparsity bias}: the noise of stochastic gradient descent drives networks toward simpler subnetworks \citep{chen2023stochastic}, and randomly initialized networks contain small subnetworks that perform comparably when trained alone \citep{frankle2018the}. Other works trace a preference for simple functions to the architecture or the parametrization \citep{valleperez2019deep,mingard2025deep}. Our notion of
\emph{proximity bias} depends only on the geometry of the neural model and is one possible source of implicit bias.

\paragraph{Singularities and learning.}\emph{Singular learning theory} studies models, including neural networks, that are not identifiable and whose Fisher information matrix is degenerate \citep{watanabe2009algebraic,wei2022deep}. Such singularities can cause plateaus in gradient-based learning \citep{amari2006singularities}, although these effects do not imply that singular points are local minimizers. These are singularities of the parametrization. We instead consider singularities and boundaries of the neural model itself, in the spirit of neuroalgebraic geometry \citep{marchetti2025position}, and our descriptions of neural models build on \citet{kileel2019expressive,kubjas2024geometry,HMK,grosdos2025algebraicgeometryrationalneural,alexandr2026attention}.

\paragraph{Subnetworks and critical points.}
Two recent works connect sparsity bias to singularities. \citet{shahverdi2026learningrazorsedgeidentifiability} study polynomial MLPs and show that, under mild conditions, subnetworks are singular points of the neural model and are critical points of the loss for an open set of targets. \citet{kohn2026singular} study deep monomial networks of high degree and show that the parametrization is critical exactly at subnetworks. However, critical points may be saddles. We ask when these submodels contain local minima for an open set of targets.

\section{Proximity bias}

\begin{figure*}[t]
    \centering
    \begin{subfigure}[t]{0.24\textwidth}
        \centering
        \includegraphics[width=\linewidth]{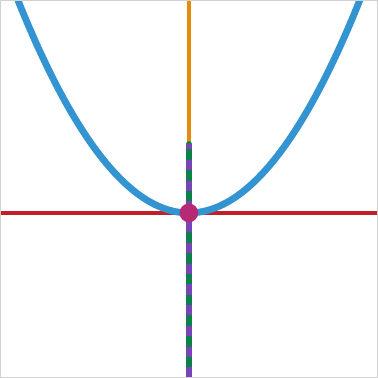}
        \caption{Smooth}
    \end{subfigure}
    \hfill
    \begin{subfigure}[t]{0.24\textwidth}
        \centering
        \includegraphics[width=\linewidth]{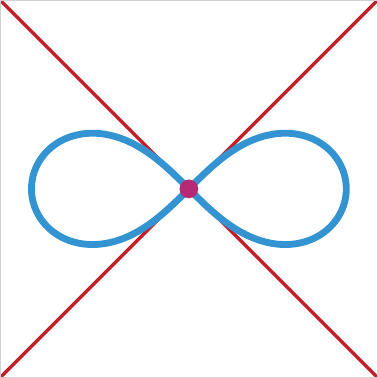}
        \caption{Node}
    \end{subfigure}
    \hfill
    \begin{subfigure}[t]{0.24\textwidth}
        \centering
        \includegraphics[width=\linewidth]{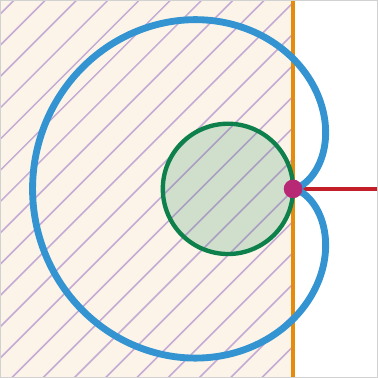}
        \caption{Cusp}
    \end{subfigure}
    \hfill
    \begin{subfigure}[t]{0.24\textwidth}
        \centering
        \includegraphics[width=\linewidth]{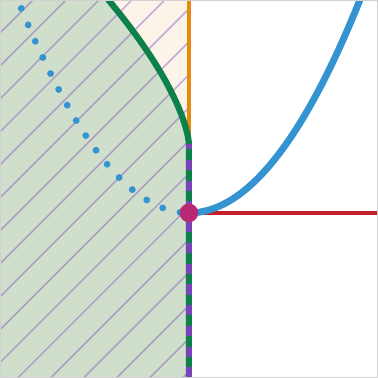}
        \caption{Boundary}
    \end{subfigure}
\caption{
\textbf{The presence or absence of proximity bias depends on local geometry.} For curve \textcolor{Curve}{$\mathcal{F}$}, there is proximity bias toward \textcolor{Origin}{$\mathcal{F'} =  \{\mathbf{0}\}$} in (c) and (d), but not in (a) and (b), as shown by whether the local Voronoi cell \textcolor{Biased}{$\Vorloc_{\mathcal F}(\mathbf 0)$} (purple dashes) is full-dimensional. 
We also draw the tangent cone \textcolor{Tangent}{$T_{\mathcal F}(\mathbf 0)$}, normal cone \textcolor{Normal}{$N_{\mathcal F}(\mathbf 0)$},~and Voronoi cell \textcolor{Voronoi}{$\Vor_{\mathcal F}(\mathbf 0)$} for comparison. The blue dots in~(d) are points in the Zariski closure of \textcolor{Curve}{$\mathcal{F}$}.
Details in \ref{sec:fig-examples}.}
    \label{fig:implicit-bias-examples}
\end{figure*}

We introduce a geometric framework to study the presence or absence of proximity bias toward submodels.
A neural model $\cF$ is a subset of a real Hilbert space of functions $\cH$.
For a fixed target function $h\!\in\!\cH$, we consider the squared distance
\[
\ell_h : \cH \to \RR, f \mapsto \ell_h(f) = \|h - f \|^2.
\]
We ask whether minimizing $\ell_h$ over $\cF$ selects a
distinguished lower-dimensional submodel $\cF'\subset\cF$. Such a preference cannot arise through smooth interior points of $\cF$; it can only occur where the local real geometry is nongeneric, for example along singular or boundary strata; see \Cref{prop:equal-dimension-bias}.
One may think of the target function $h \in \cH$ as an oracle that best explains the data, or as a teacher from which we try to distill knowledge using students in $\cF$ \citep{hinton2015distilling}.

The \emph{Voronoi cell} \citep{voronoi-cells} of $f \in \cF$ is 
\begin{align*}
    \Vor_\cF(f) &= \{h \in \cH : \| h - f \| \leq \| h - g\| \text{ for all } g \in \cF  \} \\
&= \{h \in \cH : f \text{ is a global minimizer of } \ell_h \rvert_\cF \}.
\end{align*}
We introduce the \emph{local Voronoi cell} of $f \in \cF$ as
\[
\Vorloc_\cF(f) = \{h \in \cH : f \text{ is a local minimizer of } \ell_h \rvert_\cF \}.
\]

More generally, given a set $\cF' \subseteq \cF$, we define its Voronoi cell and local Voronoi cell, respectively, as
\[
\Vor_\cF(\cF') = \!\!\bigcup_{f \in \cF'} \!\!\Vor_\cF(f), \quad \Vorloc_\cF(\cF') = \!\!\bigcup_{f \in \cF'}\!\! \Vorloc_\cF(f).
\]
We say that $\cF$ exhibits \emph{proximity bias toward $\cF'$} if $\Vorloc_{\cF}(\cF')$ has nonempty interior in $\cH$. 
When~$\cH$ is finite-dimensional and $\cF$ and $\cF'$ are semialgebraic sets (e.g., for polynomial neural networks), the set~$\Vorloc_{\cF}(\cF')$ is semialgebraic and the following are equivalent: $\Vorloc_{\cF}(\cF')$ has nonempty interior in $\cH$; $\Vorloc_{\cF}(\cF')$ has a positive Lebesgue measure;  $\dim \Vorloc_{\cF}(\cF') = \dim \cH$.

\begin{remark}
    Voronoi cells are closed convex sets that tessellate the ambient space. Local Voronoi cells are also convex, but may not be closed and may overlap.
\end{remark}

    We study (local) Voronoi cells using classical notions from variational geometry \citep{rockafellar1998variational}.
    Tangent cones and their polar normal cones encode the first-order geometry relevant to local optimality;
    see \Cref{fig:implicit-bias-examples}.
    Given $f\in\cF$, the \emph{tangent cone} of $\cF$ at~$f$, denoted $T_\cF(f)$, is the set of directions that arise as limits of sequences in $\cF$ tending to $f$. In symbols,
    $$
    T_\cF(f)
    =
    \left\{
    v\in\cH:
    \begin{array}{l}
    \exists (f_n)_{n}\subset\cF, \; \exists (\tau_n)_{n} \subset \RR_{> 0} \\
    f_n\to f,\; 
    (f_n-f)/\tau_n\to v
    \end{array}
    \right\}.
    $$
    The \emph{normal cone} of $\cF$ at $f$ is the polar cone of $T_\cF(f)$:
    $$
    N_\cF(f)
    =
    \{w\in\cH:
    \langle v,w\rangle\leq0
    \text{ for all }v\in T_\cF(f)\}.
    $$
    When $f$ is a smooth point of $\cF$, we recover the classical notions of tangent and normal spaces. The following chain of inclusions relates these sets.

\begin{restatable}{proposition}{implicitbiascriterion}\label{prop:implicit-bias-criterion}
    Given any $\cF' \subseteq \cF$ one has
\[
\cF' \subseteq
\operatorname{Vor}_{\cF}(\cF')
\subseteq
\Vorloc_{\cF}(\cF')
\subseteq
\bigcup_{f\in\cF'}\bigl(f+N_\cF(f)\bigr).
\]
\end{restatable}

\begin{proof}[Proof idea]
The first inclusion follows from $\ell_f(f) = 0$ and $\ell_f\geq 0$. The second inclusion says that global minimizers are local minimizers. The third uses that $- \nabla \ell_h(f) = 2(h - f)$ and that this vector lies in $N_\cF(f)$ if $f$ is a local minimizer~of~$\ell_h \rvert_\cF$. See Appendix~\ref{sec:proof-toolbox}.
\end{proof}
The above result uses the specificity of the normal cone to minimizers rather than any critical point.
\begin{corollary}
    If $\bigcup_{f\in\cF'}\bigl(f+N_\cF(f)\bigr)$ has empty interior, then there is no proximity bias toward $\cF'$. If~$\Vor_\cF(\cF')$ has nonempty interior, then there is proximity bias toward~$\cF'$.
\end{corollary}

There is always proximity bias toward equal-dimensional submodels  and proximity bias toward a lower-dimensional submodel requires points where the normal cone is larger than at smooth points.
\begin{restatable}{proposition}{equaldimensionbias}\label{prop:equal-dimension-bias}
    Assume $\dim \cH < \infty$ and consider semialgebraic sets $\cF' \subseteq \cF \subseteq \cH$. If $\dim\mathcal F'=\dim\mathcal F$, then there is proximity bias toward~$\mathcal F'$. If $\dim \cF' < \dim \cF$ and $\cF'$ is contained in the smooth locus of $\cF$, then there is no proximity bias toward $\cF'$.
\end{restatable}
\begin{proof}[Proof idea]
    First, a full-dimensional semialgebraic subset $\cF'$ of $\cF$ agrees with $\overline \cF$ near some $f \in \cF'$, so every target close to $f$ has all its nearest points of $\cF$ in $\cF'$. Second, local Voronoi cells of smooth points have dimension at most $\dim \cH - \dim \cF < \dim \cH - \dim \cF'$.
\end{proof}

\begin{remark}[Other losses]
    Our setting can be generalized to other \emph{nondegenerate losses}: Fr\'echet differentiable functions $\ell_u : \cH \to \RR_{\geq0}$ such that $\ell_u(u) \!= \!0$ and $\nabla\ell_u(f) = 0$ if and only if $f = u$.
    The squared distance $\ell_u(f) = \| f-u \|^2$ is a nondegenerate loss because $\nabla \ell_u(f) = 2(f-u)$. If~$\cH$ is a finite-dimensional space, then the mean-squared error (MSE) is a nondegenerate loss function provided that the dataset is sufficiently general and large enough \citep{kohn2022geometry}.
\end{remark}

\section{Multilayer Perceptrons}\label{sec:mlp}

Figure~\ref{fig:rank-one-cone} conveys the idea behind this section. It shows the cone of functions computed by a single neuron.
An inactive neuron computes the zero function, which is the apex of the cone. An inactive neuron can be reactivated in every direction of the cone, so submodels with inactive neurons have many tangent directions and tend not to be preferred. For shallow networks, this idea suffices to show that there is no proximity bias toward subnetworks. For deep networks, one must also take into account how the upstream and downstream layers affect the tangent directions added by inactive neurons, which makes the analysis more challenging.

\begin{definition}
A \emph{multilayer perceptron (MLP)} with architecture 
$\mathbf d = (d_0, \dots, d_L)$ and activation functions $\boldsymbol \rho = (\rho_1, \ldots, \rho_L)$ is a map 
\[ f_{\mathbf W, \mathbf b} = \tilde \rho_L \circ A_L \circ   \cdots \circ \tilde \rho_1 \circ A_1 : \RR^{d_0} \to \RR^{d_L}, \]
where $A_k(x) = W_kx + b_k$ with $W_k \in {\RR^{d_k \times d_{k-1}}}$, $b_k \in \RR^{d_k}$, and $\tilde \rho_k : \RR^{d_k} \to \RR^{d_k}$ applies a fixed function $\rho_k : \RR \to \RR$ coordinatewise. 
The \emph{neural model} $\cF$ of the MLP is the set of functions $f_{\mathbf W, \mathbf b}$ as the \emph{weight matrices} $\mathbf W = (W_1, \dots, W_L)$ and \emph{biases} $\mathbf{b} = (b_1, \dots, b_L)$ vary over $\prod_k\RR^{d_k \times d_{k-1}}$ and $\prod_k\RR^{d_k}$, respectively. 
The hyperparameter $L$ is the number of layers. We say that an MLP is \emph{shallow} if~$L=2$, and \emph{deep} if $L \geq 3$. 
\end{definition}

For simplicity, we assume linear readouts; i.e., the last activation function $\rho_L$ is the identity map. 
This assumption is mild, since the last activation can be absorbed by the loss function: $\ell_h(\rho \circ f) = (\ell_h \circ \rho)(f)$.
We assume that all other activations are nonconstant.

Given a fixed MLP architecture and activation functions, we take the ambient space $\cH$ of the neural model $\mathcal F$ to be the completion of its linear span 
with respect to a fixed inner product: $\cH = \overline{\linspan(\cF)}$. When all activations are polynomials, $\cF$ consists of (tuples of) polynomials of bounded degree, so $\cH$ is finite-dimensional and we take the standard inner product of polynomials (Bombieri inner product).
When some activation is non-polynomial, the universal approximation theorem implies that $\linspan(\cF)$ is infinite dimensional, and we take an $L^2$ inner product with respect to a finite compactly-supported Borel measure~$\mu$ on $\RR^{d_0}$, so that $\cH = \overline{\linspan(\cF)} = L^2(\mu, \RR^{d_L})$. This implies restricting the support of elements in $\cF$ to guarantee that their norm is finite. When the support of $\mu$ is a finite set $\Omega$, the functions in the restricted neural model $\mathcal{F}|_{\Omega}$ are identified with vectors in $\RR^{|\Omega|\times d_L}$, a finite-dimensional ambient space. 
Equipping $\RR^{|\Omega|\times d_L}$ with the Euclidean inner product corresponds to MSE for the dataset $\Omega$.

\subsection{Shallow Networks}
\label{sec:shallow}

Consider MLPs with architecture $\mathbf d = (d_0, d_1, d_2)$ and activations $\boldsymbol{\rho} = (\tilde \rho, \id)$. The \emph{single-neuron cone} is
\[
\mathcal C \!= \!\left\{f(x) = \rho(\langle w,x\rangle + \beta)  v + b : \!\begin{array}{l}
    w \in \RR^{d_0}, \beta \in \RR \\
     v, b \in \RR^{d_2}
    \end{array}
 \!\!\right\}.
\]
The single-neuron cone is the neural model of MLPs with architecture $(d_0, 1, d_2)$ and activations $\boldsymbol{\rho} = (\tilde \rho, \id)$: $f(x)$ above is $f_{\mathbf W, \mathbf{b}}$ with $\mathbf{W} = (w^\top, v)$ and $\mathbf{b} = (\beta, b)$.
\begin{lemma}\label{lem:shallow-sum-cones}
    The neural model $\mathcal F$ of a shallow MLP with architecture $(d_0, d_1, d_2)$ is the Minkowski sum of $d_1$ single-neuron cones:
    $
    \mathcal{F} = \sum\nolimits_{i = 1}^{d_1} \mathcal C.
    $
    Therefore, $\linspan(\cC) = \linspan(\cF)$.
\end{lemma}

We think of nonzero elements in $\cC$ as rank-one functions and elements in $\cF$ as rank at most $d_1$ functions, by analogy to the following special case.
\begin{example}[Symmetric matrices] Consider the architecture $\mathbf{d} \!= \!(d_0, d_1, 1)$ and $\boldsymbol{\rho} = (\tilde \rho, \id)$ with $\rho (t) = t^2$. The MLPs without bias terms $(b_k = 0)$ and weights $\mathbf{W} = (W, \lambda)$ are quadratic forms 
\[
f_{\mathbf{W}, \mathbf{0}}(x) = x^\top Ax, \quad A = \sum\nolimits_{i=1}^{d_1} \lambda_i w_iw_i^\top\]
where $w_i$ is the $i$-th row of $W$.
Each hidden neuron contributes a rank-one term $\lambda_i w_iw_i^\top$, so $\rank A \leq d_1$. Let $\cF$ be the set of $d_0 \times d_0$ symmetric matrices of rank at most $d_1$, where $d_1 < d_0$ and let $\cF' \subset \cF$ be its singular locus: points where the rank drops below $d_1$. As recalled above, the Eckart-Young theorem implies that a best rank-at-most-$r$ approximation of a matrix $M$ has rank $r$ (not smaller) unless $\rank(M) < r$. That is, $\Vor_\cF(\cF') = \cF'$.
\end{example}

Our next result extends this matrix example in three directions: to all local minimizers (not only global ones), to every inner product and nondegenerate loss, and to every activation function. Informally, we show that given a target function $h$, a locally best rank-at-most-$d_1$ approximation of $h$ has rank~$d_1$ (not smaller) unless $\rank h < d_1$. See \Cref{fig:mlp}.

\begin{figure}
    \centering
    \includegraphics[width=0.5\linewidth]{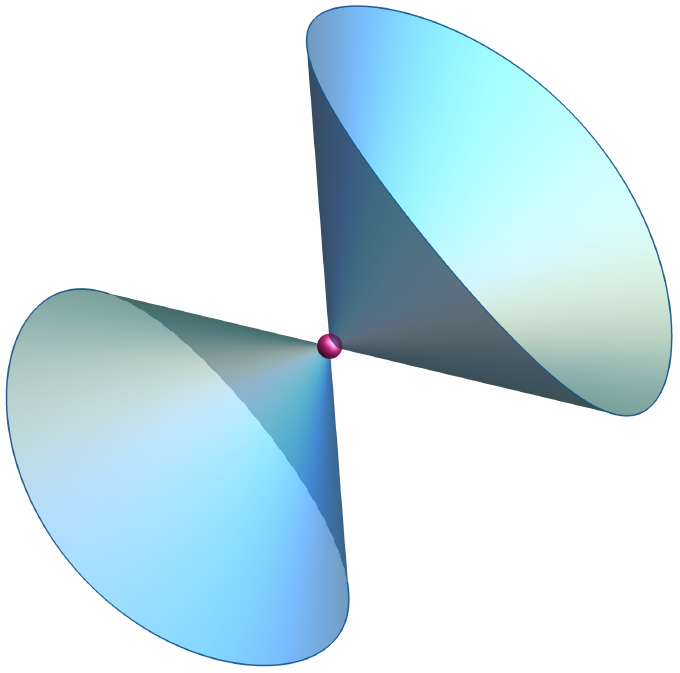}    
    \caption{
    \textbf{No proximity bias toward the origin for single-neuron cones}: $ \Vorloc_{\textcolor{Curve}{\cC}}(\textcolor{Origin}{0}) = N_{\textcolor{Curve}{\cC}}(\textcolor{Origin}{0}) = \{\textcolor{Origin}{0}\}$.}
    \label{fig:rank-one-cone}
\end{figure}

\begin{restatable}{theorem}{nobiasshallow}\label{thm:no-bias-shallow}
    Let $\cF$ be the neural model of
        a shallow MLP with architecture $\mathbf d = (d_0, d_1, d_2)$, $d_1 \geq 2$, and
        activations $(\tilde \rho, \operatorname{id})$. Let $\cH$ be the closure of $\linspan(\cF)$ with respect to a fixed inner product, and let $\cF'$ be the submodel with architecture $\mathbf d' = (d_0, d_1-1, d_2)$.
        Then, 
        \[
        \Vor_\cF(\cF') =  \Vorloc_\cF(\cF') = \cF'.
        \]
        That is, a target function $h \in \cH$ admits a (local) minimizer of $\ell_h\rvert_\cF$ in $\cF'$ if and only if $h \in \cF'$. The same holds with $\cF$ and $\cF'$ replaced by their closures.
\end{restatable}
    \begin{proof}[Proof idea]
        We have $\cF' = \sum_{i=1}^{d_1-1} \mathcal C$ and $\cF = \cF' + \cC$, by \Cref{lem:shallow-sum-cones}. Given any $f \in \cF'$ its tangent cone at~$\cF$ contains~$\cC$, which spans the whole ambient space. Hence, the normal cone $N_\cF(f)$ is zero (see \Cref{fig:rank-one-cone}) and the conclusion follows from \Cref{prop:implicit-bias-criterion}.
    \end{proof}
    Details of all proofs in this section are in Appendix~\ref{sec:proofs-mlp}, and details of all figures are in Appendix~\ref{sec:figures}.
    \begin{corollary}
        Under the setup from \Cref{thm:no-bias-shallow}, there is no proximity bias toward $\cF'$ whenever $\cF'$ has empty interior in $\cH$. Thus, there is also no proximity bias toward (even) smaller subnetworks.
    \end{corollary}

\begin{remark}[Smallest Voronoi cells]
    The equalities $\Vor_\cF(\cF') = \Vorloc_\cF(\cF') = \cF'$ in \Cref{thm:no-bias-shallow} mean the cells are as small as possible, by \Cref{prop:implicit-bias-criterion}.
\end{remark}

    \begin{remark}[Saddle points in parameter space]
        Our results pertain to function space, but have implications in parameter space. Inactive neurons (i.e., subnetworks) lead to critical points of the loss in parameter space for a nonempty open set of target functions~\citep{kohn2026singular}. Our tangent cone computation implies that these critical points cannot be local extrema at parameters where the parametrization is locally open (unless the loss vanishes there), so they must be saddle points.
    \end{remark}
    \begin{remark}[Bounding normal cones]\label{rem:bound-normal-parametrization}
    Conversely, results in parameter space can have implications in function space. Given a parametrization~$\phi: \Theta \to \cF$, we have $N_{\cF}(\phi(\theta)) \subseteq \im(d\phi \rvert_ \theta)^\perp$, the orthogonal complement of the image of the differential of $\phi$ at~$\theta$.
    This containment can sometimes be used to deduce that there is no proximity bias, but not always (e.g. for MLPs) since we may have strict inequalities $\dim N_{\cF}(\phi(\theta)) < \dim  \im(d\phi \rvert_ \theta)^\perp$.
    In Appendix~\ref{sec:cnn-proof} we use this to show that there is no proximity bias toward subnetworks of convolutional neural networks (CNNs) with polynomial activations.
    \end{remark}

\subsection{Deep Networks}\label{sec:deep}

In this section, we consider deep MLPs.
We start by analyzing proximity bias in the linear case.

\begin{example}[Deep linear networks]
    If all activations are the identity, the neural model of architecture $(d_0, \dots, d_L)$ consists of affine maps $\RR^{d_0} \to \RR^{d_L}, x \mapsto Wx + b$ with~$\rank W \leq r := \min_k d_k$, so it is the neural model of the shallow linear network $(d_0, r, d_L)$. If we remove a neuron from a bottleneck layer $k$ (i.e., one with $d_k = r$), there is no proximity bias toward the corresponding submodel, by \Cref{thm:no-bias-shallow}. Otherwise, deleting a neuron of a hidden layer does not change the model ($\mathcal F' = \mathcal F$), so there is proximity bias toward~$\cF'$.
\end{example}

Unlike the shallow case, neural models of deep MLPs can no longer be expressed as a sum of single-neuron cones, in general. However, for some neural submodels $\cF' \subseteq \cF$ we can find a cone $\tilde \cC$ whose span is dense in~$\cH$ such that $\cF' + \tilde \cC \subseteq \cF$.
This inclusion is enough to conclude that there is no proximity bias toward $\cF'$.

Next, we show that there is no proximity bias toward subnetworks with inactive neurons in different situations: when one neuron is inactive in every hidden layer, when a neuron of the first hidden layer is inactive and the subnetwork is small enough, when a neuron of the last hidden layer is inactive and the earlier layers are wide enough, and, for analytic activations or polynomial activations of sufficiently high degree, when a neuron of any hidden layer is inactive.

\paragraph{All hidden layers.} The idea behind \Cref{thm:no-bias-shallow} can be generalized to deep MLPs when our submodel is defined by having an inactive neuron in every hidden layer, which we call an inactive path (cf. \cite{alexandr2025relu}).

\begin{restatable}[No bias toward inactive paths]{theorem}{inactivepaths}\label{thm:inactive-paths}
    Let $\mathcal{F}$ be the neural model of an MLP with architecture $\mathbf d = (d_0, d_1, \dots, d_{L-1}, d_L)$ and let $\mathcal{F}'$ be the submodel with architecture $(d_0, d_1-1, \dots, d_{L-1}-1, d_L)$. Assume the activations are continuous. Then,
    \[
    \Vor_\cF(\cF') = 
    \Vorloc_{\mathcal{F}}(\mathcal F') = \mathcal F'.
    \]
\end{restatable}

\begin{proof}[Proof idea]
Using the inactive path we construct tangent directions $\rho_{L-1} \circ \cdots \circ \rho_{1}(\langle w, x \rangle + \beta) \; v$ for all $f \in \cF'$, and these directions span a dense subset of $\cH$. 
\end{proof}

\paragraph{Last hidden layer.} We use the observation that a deep MLP  is a shallow MLP whose input space is a space of functions rather than the vector space $\RR^{d_0}$. For a parameter $\theta=(\mathbf W,\mathbf b)$ let
\[
h_\theta=\rho_{L-2}\circ A_{L-2}\circ\dots\circ\rho_1\circ A_1\colon\ \mathbb R^{d_0}\to\mathbb R^{d_{L-2}}
\]
be the \emph{feature map} computed by the first $L-2$ layers ($h_\theta(x)=x$ for $L=2$). Fixing these layers gives a shallow MLP on the features $h_\theta$:
\[
f_\theta(x)=b_L+\sum_{i=1}^{d_{L-1}}\rho_{L-1}\bigl(\langle w_i,h_\theta(x)\rangle+\beta_i\bigr)\,v_i ,
\]
where $w_i$ is the $i$-th row of $W_{L-1}$, $v_i$ the $i$-th column
of $W_L$, and
$\beta_i$ the $i$-th entry of $b_{L-1}$. 
For fixed weights in the first $L-2$ layers, a neuron in the last hidden layer computes $\rho_{L-1}\circ q$ with $q$ in the space of functions $V_\theta=\operatorname{span}\{1,h_{\theta,1},\dots,h_{\theta,d_{L-2}}\}$ spanned by the features, just as a neuron of a shallow network computes $\rho\circ q$ with $q\in\operatorname{span}\{1,x_1,\dots,x_{d_0}\}$: the deep network is a shallow network whose input space is the function space $V_\theta$. In the language of kernel methods, $V_\theta$ is the reproducing kernel Hilbert space of the kernel $k_\theta(x,x')=1+\langle h_\theta(x),h_\theta(x')\rangle$.

This viewpoint explains what happens when neurons in the last hidden layer $(L-1)$ are inactive.
Define
\[
\cC_\theta =  \!\{\rho_{L-1}(\langle w, h_\theta\rangle + \beta) v  : w \in \RR^{d_{L-2}}, \beta \in \RR, v \in \RR^{d_L}  \!\},
\]
the set of functions computed by a single neuron of the last hidden layer with input features $h_\theta$, which is a symmetric cone.
We say that $\theta$ has \emph{complete features} 
if the span of $\cC_\theta$ is dense in $\cH$, that is, the deep network has the same expressivity as a shallow network defined on a function space. For polynomials, this occurs when hidden layers are wide enough (see \Cref{cor:mlp-last-layer-polynomials}). In the following, a representation of $f \in \cF'$ is a subnetwork parameter~$\theta$ with $f_\theta = f$.

\begin{restatable}{proposition}{completefeaturesorthogonal}\label{thm:complete-features-orthogonal}
    Let $\cF$ be the neural model of an MLP with arbitrary activations,
$L\ge3$ layers, $d_{L-1}\ge2$ and let $\cF'$ be the subnetwork with one neuron of the last hidden layer inactive.
If every $f \in \mathcal F'$ has a representation with complete features, then 
$\Vor_\cF(\cF') =  \Vorloc_{\mathcal F}(\cF') = \cF'$.
\end{restatable}

\begin{proof}[Proof idea]
    The deleted neuron of the last hidden layer can be put back with arbitrary incoming weights $w$ and arbitrarily small outgoing weights $v$ of either sign. This gives lines in $\cF$ through $f$ in the directions $\rho_{L-1}(\langle w, h_\theta \rangle + \beta) v \in \cC_\theta$. So $N_\cF(f)$ is orthogonal to $\linspan \cC_\theta$, and it is zero for complete features.
\end{proof}

\begin{restatable}{corollary}{mlplastlayerpolynomials}\label{cor:mlp-last-layer-polynomials}
    Under the assumptions of \Cref{thm:complete-features-orthogonal} with nonconstant polynomial activations of degree $r_1, \dots, r_{L-1}, 1$, if 
    $
    d_k \geq \binom{d_0 + r_1 \cdots r_k}{d_0} -1 
    $
    for all $k = 1, \dots, L-2$,
    then $\Vor_\cF(\cF') = \Vorloc_{\mathcal F}(\cF') = \cF'$. 
\end{restatable}

\begin{proof}[Proof idea]
    If the first $L-2$ layers are wide enough, they can compute a basis of all polynomials of degree at most $r_1 \cdots r_{L-2}$.
    Every function of $\cF'$ can be computed in this basis, and then a single new neuron in layer $L-1$ gives all polynomials of degree at most $r_1 \cdots r_{L-1}$, so the features are complete.
\end{proof}

\paragraph{First hidden layer.} The question of what happens when a neuron is deleted in the first layer reduces to the shallow setting of Section~\ref{sec:shallow}, but with a different loss.
Cut the MLP after the affine map of layer 2: $f_{\mathbf{W}, \mathbf{b}} = \Psi \circ z$, where $z = W_2 \tilde{\rho}(W_1 x + b_1) + b_2$ is a shallow MLP with architecture $(d_0, d_1, d_2)$ and $\Psi$ is the map computed by the remaining layers. This gives the pulled-back loss $\tilde{\ell}_h(z) = \ell_h(\Psi \circ z)$, which may be degenerate. Indeed, its gradient $\nabla \tilde \ell_h(z) = J^\top \nabla \ell_h(f)$, with $J = D\Psi \circ z$, vanishes whenever $\nabla \ell_h(f) \in \ker J^\top$. Local Voronoi cells
may not be trivial, but we can exclude proximity bias when $\mathcal{\cF'}$ is small enough.

\begin{restatable}{proposition}{mlpfirstlayerpoly}\label{thm:mlp-first-layer-poly}
    Let $\cF$ be the neural model of a deep MLP with architecture $\mathbf{d} = (d_0, \dots, d_L)$ and nonconstant polynomial activations $\boldsymbol \rho$. Let $\deg \rho_1 = r_1$ and let $\mathcal F'$ be the submodel with width
$d_1-1$ in the first hidden layer and all other layers unchanged. 
If
\[
\dim \mathcal F' < \binom{d_0 + r_1}{d_0} = \dim \RR[x_1, \dots, x_{d_0}]_{\leq r_1}
\]
then there is no proximity bias toward $\mathcal F'$.
\end{restatable}

\begin{proof}[Proof idea]
    The reinserted neuron gives more linearly independent directions than the dimension of $\cF'$, so the union of translated normal cones from \Cref{prop:implicit-bias-criterion} has positive codimension.
\end{proof}

\paragraph{Any hidden layer.}
A neuron of hidden layer $k$ combines both situations. It reads the features computed by the previous layers, as in the last layer, and its output is processed by the subsequent layers, as in the first layer. Cutting after the affine map of layer $k+1$, $f_\theta = \Psi \circ z$, where $z$ is a shallow network on the features $h_\theta^{(k-1)}$ trained with the pulled-back loss $\tilde \ell_h(z) = \ell_h(\Psi \circ z)$.

    For polynomial MLPs, there is no proximity bias toward a submodel $\mathcal{F}'$ with width $d_{k}-1$ in a hidden layer $k$, provided the $d_j$ are sufficiently high, 
by \Cref{thm:mlp-first-layer-poly} and \Cref{cor:mlp-last-layer-polynomials}. 
    The width conditions are that
    $d_j \ge \binom{d_0 + r_1 \cdots r_j}{d_0} - 1$ for all $j < k$ and we also require $\dim \cF' < \binom{d_0+r_1\cdots r_k}{d_0}$.

These width conditions can be dropped when the activation in layer $k$ has sufficiently high degree.

\begin{restatable}{proposition}{anylayerpoly}\label{prop:any-layer-poly}
    Let $\cF$ be the neural model of a deep MLP with architecture $\mathbf{d} = (d_0, \dots, d_L)$ and nonconstant polynomial activations $\boldsymbol \rho$, and let $\cF'$ be the submodel with width $d_k - 1$ in a hidden layer $k$. If $\deg \rho_k \geq \dim \cF'$, there is no proximity bias toward $\cF'$.
\end{restatable}
\begin{proof}[Proof idea]
    The functions computed by a single neuron in layer $k$ span all polynomials of degree at most $\deg \rho_k$ in the features, which gives $\deg \rho_k + 1 > \dim \cF'$ linearly independent constraints on the normal cones.
    Hence, a dimension count as in \Cref{thm:mlp-first-layer-poly} concludes the proof.
\end{proof}
For non-polynomial activations, no condition on the dimension of $\cF'$ is needed.

\begin{restatable}{theorem}{mainnonpolyanyhiddenlayer}\label{thm:main-nonpoly-any-hidden-layer}
    Given a compactly supported finite Borel measure $\mu \neq 0$ on $\RR^{d_0}$ with a Lebesgue density, let $\mathcal F \subseteq L^2(\mu, \RR^{d_L})$ be the neural model of an MLP with architecture $\mathbf{d} = (d_0, \dots, d_L)$ and analytic nonconstant activations. Let $\mathcal F'$ be the submodel given by the subarchitecture with width $d_k - 1$ in a hidden layer~$k$, and suppose that $\rho_k$ is not polynomial. Then, there is no proximity bias toward $\cF'$.
    Moreover, if $k=1$ and $d_L = 1$, then $\Vor_\cF(\cF') = \Vorloc_{\mathcal F}(\cF') = \cF'$.
\end{restatable}

\begin{proof}[Proof idea]
For analytic nonpolynomial $\rho_k$, a single neuron can approximate every function of the features, and $J$ vanishes only on a measure zero set. The normal cones need not vanish, but they satisfy more linearly independent constraints than $\cF'$ has parameters.
Dimension counts on finite-dimensional slices and the Baire category theorem conclude the proof.
\end{proof}

It remains to understand what happens when we have an inactive neuron in an arbitrary hidden layer for polynomial neural networks with low activation degrees. 
In some cases, we may obtain a proper submodel $\cF' \subsetneq \cF$ with $\dim \cF' = \dim \cF$ (e.g., \Cref{ex:rank-two-three}), in which case there is proximity bias by \Cref{prop:equal-dimension-bias}.
We conjecture that having a submodel of smaller dimension is sufficient to rule out proximity bias, and provide numerical evidence in Appendix~\ref{sec:numerical-evidence-conjecture}. 
\begin{conjecture}\label{conj:mlp-small-dim-no-bias}
    Let $\cF$ be the neural model of an MLP with polynomial activations, and consider the submodel $\mathcal F' \subset \cF$ obtained by making some hidden neurons inactive. If $\dim\mathcal F'<\dim\mathcal F$, then there is no proximity bias toward $\mathcal F'$.
\end{conjecture}

\section{Lightning Self-Attention}
\begin{figure*}[t]
    \centering

    \begin{subfigure}[t]{0.32\textwidth}
        \centering
        \includegraphics[width=\linewidth]{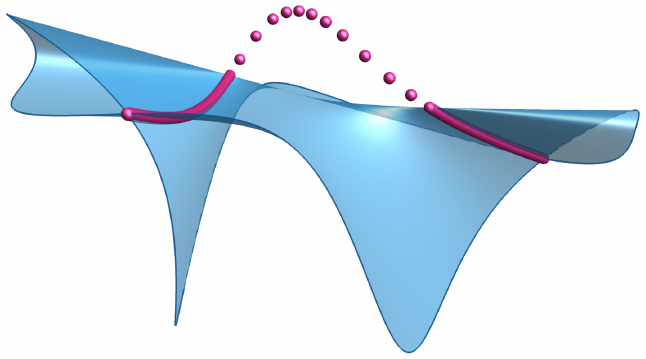}
        \caption{$a = d=2$}
        \label{fig:sub1}
    \end{subfigure}
    \hfill
    \begin{subfigure}[t]{0.32\textwidth}
        \centering
        \includegraphics[width=0.5\linewidth]{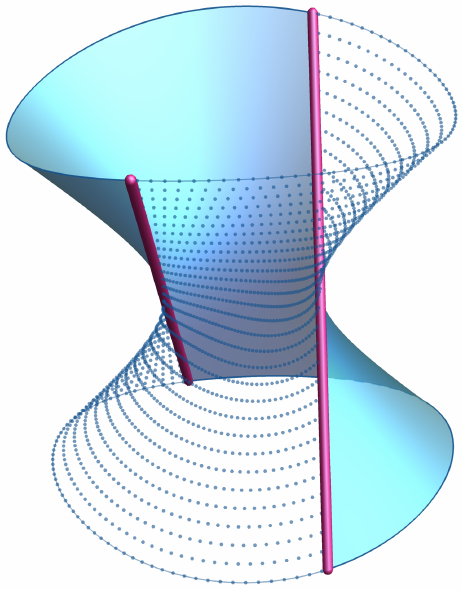}
        \caption{$a=1$, $d=2$}
        \label{fig:sub2}
    \end{subfigure}
    \hfill
    \begin{subfigure}[t]{0.32\textwidth}
        \centering
        \includegraphics[width=0.5\linewidth]{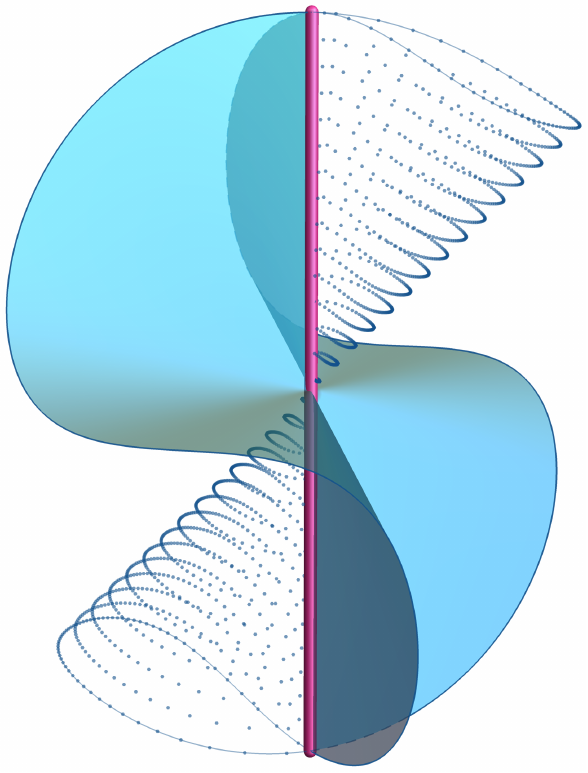}
        \caption{$a=1$, $d=3$}
        \label{fig:sub3}
    \end{subfigure}

    \caption{
    \textbf{No proximity bias toward submodels of multi-token self-attention mechanisms}, except for a rank-one scalar output model $(d' = a=1)$; see \Cref{thm:multi-token-lower-rank,thm:multi-token-rank-one}. Submodels are the magenta loci, and the dotted parts represent points in the Zariski closure of the respective neural models. Details in \ref{sec:fig-attention}.
    }
    \label{fig:attention-rank-one-geometry}
\end{figure*}

We next study proximity bias in lightning self-attention, following
the setup of \citet{HMK}. Let $X\in\RR^{d\times t}$ contain $t$
input tokens, and let $Q,K\in\RR^{a\times d}$ and
$V\in\RR^{d'\times d}$ be the query, key, and value matrices.
Softmax self-attention takes the form
$$
X\longmapsto VX\operatorname{softmax}\bigl(X^\top K^\top QX\bigr),
$$
where softmax is applied columnwise.
Lightning self-attention omits the softmax normalization, giving
$$
X\longmapsto VXX^\top K^\top QX.
$$
Since $K$ and $Q$ enter only through $A=K^\top Q$, we
reparametrize the model as
$$
\phi(A,V)(X)=VXX^\top AX,
\qquad \rank(A)\leq a,
$$
and write
$$
\mathcal X_a
=
\{\phi(A,V):\rank(A)\leq a,\ V\in\RR^{d'\times d}\}.
$$
Each output coordinate is a homogeneous cubic polynomial in the
$dt$ entries of $X$. Throughout our analysis of lightning attention,
the ambient space is therefore
$
\cH=\RR^{d'\times t}\otimes\RR[x_1,\dots,x_{dt}]_3,
$
equipped with a fixed inner product.

For $s<a$, let $\mathcal X_s\subseteq\mathcal X_a$ be the submodel
consisting of functions in $\mathcal X_a$ that have a parametrization with $\rank(A)\leq s$. 
Thus,
$
\mathcal X_1\subseteq\mathcal X_2\subseteq\cdots\subseteq\mathcal X_a.
$
Our goal is to determine whether there is proximity bias toward these lower-rank submodels or
singular loci of the model. 
The singular locus need not be a submodel obtained
by lowering the rank constraint, so we consider singularities separately.
The answer depends on the number of tokens
and on the rank. For multiple-token attention, there is no proximity bias toward any proper
lower-rank submodel. The only multiple-token proximity bias
occurs in the case $a=d'=1$. For
single-token attention, the rank constraint creates
boundary strata that give proximity bias.

\subsection{Multiple-Token Attention}

The multiple-token case is $t>1$. We assume without loss of generality
that $a\leq d$, since for $a>d$ the rank constraint is vacuous and
$\mathcal X_{a-1}=\mathcal X_a$. For $a>1$, there is no proximity bias
toward any proper lower-rank attention submodel.

\begin{restatable}{theorem}{multitokenlowerrank}
\label{thm:multi-token-lower-rank}
Assume $t>1$, and $1<a\leq d$. Then there is no proximity bias
toward $\mathcal X_s$ for any $s<a$.
\end{restatable}
\begin{proof}[Proof idea]
    At a function with $\rank A < a$, every rank-one change of $A$ and every change of $V$ is a two-sided tangent direction. Outside a small exceptional locus these directions span a space of dimension $\dim \mathcal X_d$, so a dimension count shows that the translated normal cones are small enough.
\end{proof}

We next ask whether there is proximity bias toward singular points of the attention model.
The algebraic description of the singular loci, obtained by
complexifying and projectivizing, is given in
Appendix~\ref{app:multi-token-attention}. For $1<a<d$, the 
real singular locus is $\mathcal X_{a-1}$. When $a=d$, it is contained
in $\mathcal X_1$. Thus, Theorem~\ref{thm:multi-token-lower-rank}
implies no proximity bias toward the real singular locus
for any $a>1$.

It remains to consider $a=1$. There is no nonzero lower-rank submodel
in this case, but the singular-locus question can still be posed. When
$d'>1$, the nonzero real singular locus is the rank-one locus
$$
Z
=
\{\phi(A,V):\rank(A)=\rank(V)=1\},
$$
and again there is no proximity bias toward it.

\begin{restatable}{proposition}{multitokenrankonenoBias}
\label{prop:multi-token-rank-one-no-bias}
Assume that $t>1$, $d>1$, $a=1$, and $d'>1$. Then there is no proximity bias toward the real singular locus of $\mathcal X_1$.
\end{restatable}
\begin{proof}[Proof idea]
    Dimension count for the singular locus~$Z$, via unique factorization and birationality.
\end{proof}

The scalar-output case $d'=1$ is different. Writing
$A=kq^\top$ and $V=v^\top$, with $k,q,v\in\RR^d$, gives
$$
\phi(A,V)(X)
=
\bigl(v^\top XX^\top k\bigr)q^\top X.
$$
Since $XX^\top$ is symmetric, the parametrization is invariant under
exchanging $v$ and $k$, and is therefore generically two-to-one.
Setting
$
S=\tfrac{1}{2}(vk^\top+kv^\top)
$
removes this symmetry and gives
$$
\phi(A,V)(X)
=
\left(\sum_{j=1}^t x_j^\top Sx_j\right)q^\top X,
$$
where $x_j$ is the $j$-th column of $X$. We therefore write $\psi(q,S)(X) = \phi(A,V)(X)$
and define
$$
U
=
\{\psi(q,S):q\in\RR^d,\ \rank(S)\leq1\}.
$$
In terms of the original parameters, $\rank(S)\leq1$ exactly when
$v$ and $k$ are linearly dependent.

\begin{restatable}{theorem}{multitokenrankone}
\label{thm:multi-token-rank-one}
Assume $t>1$, $a=d'=1$, and $d>1$. There is proximity bias
toward $U$. If $d>2$, $U$ is both the real singular locus of
$\mathcal X_1$ and a boundary of the real model. If $d=2$, $U$ is a boundary of the real~model.
\end{restatable}
\begin{proof}[Proof idea]
    A real inertia constraint makes $U$ one-sided, and the Voronoi cell of $U$ fills an open set.
\end{proof}

The two geometric regimes in
Theorem~\ref{thm:multi-token-rank-one} are visible already in the
smallest dimensions; see Figure~\ref{fig:attention-rank-one-geometry}.
By Lemma~\ref{lem:multi-token-symmetric-quotient}, the map
$q\otimes S\mapsto\psi(q,S)$ is injective for $t>1$, so these
low-dimensional slices faithfully represent the corresponding
geometry of $\mathcal X_1$. For $d=2$, the locus $U$ appears as a
boundary of the real model despite the algebraic model being smooth,
whereas for $d=3$ it appears as a singular locus.

The multiple-token case is summarized as follows.

\begin{center}
\begin{tabular}{|c|c|c|}
\hline
Regime & $a>1$ or $d'>1$ & $a=d'=1$\\
\hline
proximity bias? & No & Yes\\
\hline
\end{tabular}
\end{center}

\subsection{Single-Token Attention}

The geometry of the attention model changes when $t=1$. Writing
$S=\frac12(A+A^\top)$, we reparametrize using a symmetric matrix
$S\in S^2(\RR^d)$ and $V\in\RR^{d'\times d}$:
\[
\psi(S,V)(x)=(Vx)(x^\top Sx).
\]
For $t>1$, the fiber is controlled by the bilinear form
$y^\top Ax$, where rank one is the exceptional reducible case.
For $t=1$, this is replaced by the quadratic form $x^\top Sx$,
which is reducible over $\CC$ when $\rank(S)\leq2$. Consequently,
rank-two symmetric matrices now appear in the singular locus.

The rank constraint is also different under symmetrization.
Over $\CC$, matrices $A$ with $\rank(A)\leq a$ give the
symmetric matrices $S$ with $\rank(S)\leq\min\{2a,d\}$,
whereas over~$\RR$ they give those $S$ with at most $a$ positive
and at most $a$ negative eigenvalues. These two observations
lead to the following.

\begin{restatable}{theorem}{singletokenattention}
\label{thm:single-token-attention}
Assume that $t=1$.
\begin{enumerate}
    \item If $a=d>1$ and $d'>1$, then
    \[
    \operatorname{Sing}\bigl(\PP(\mathcal X_d^\CC)\bigr)
    =
    \PP\left(
    \left\{\psi(S,V): \begin{matrix}
        \rank(V)=1,\\ \rank(S)\leq2
    \end{matrix}\right\}
    \right),
    \]
    and there is no proximity bias toward the corresponding real
    singular locus.

    \item If $a>1$ and $2a<d$, then
    \[ 
    \operatorname{Sing}\bigl(\PP(\mathcal X_a^\CC)\bigr)
    =
    \PP\bigl(
    \{\psi(S,V):\rank(S)\leq2a-1\}
    \bigr).
    \]
    The real model $\cX_a$ exhibits proximity bias toward the rank-$(2a-1)$ stratum $U$ where $S$ has inertia
    $(a,a-1,d-2a+1)$.
\end{enumerate}
\end{restatable}
\begin{proof}[Proof idea]
    Part 1 is a dimension count as in \Cref{prop:multi-token-rank-one-no-bias}. Part 2 is a one-sided argument that implies full-dimensional Voronoi cells, as in \Cref{thm:multi-token-rank-one}.
\end{proof}

Thus, the larger singular locus in the single-token setting
does not by itself produce proximity bias. In the first regime, the
rank-two factorization locus is singular but there is no proximity bias toward it. In the
second, proximity bias arises from the one-sided boundary induced by
the real inertia constraint, paralleling the boundary mechanism in
Theorem~\ref{thm:multi-token-rank-one}. Proofs are given in
Appendix~\ref{app:single-token-attention}.

\subsection{Beyond lightning self-attention}
\label{sec:softmax-experiment}

Theorem~\ref{thm:multi-token-rank-one} identifies a
proximity bias toward~$U$ in rank-one scalar-output lightning attention.
We ask whether this proximity bias is visible under gradient-based training after replacing
lightning attention by softmax attention. The change in identifiability
caused by normalization is studied in~\citet{HMK}; here we compare the
training dynamics empirically.

We take $d=t=2$ and $a=d'=1$. In each run, the target is
$h=\psi(q_\ast,S_\ast)$, where $q_\ast\in\RR^2$ is a unit vector and
$S_\ast\succ0$ has eigenvalues $1$ and $0.2$. We sample $512$
standard Gaussian inputs $X_i$ and use the labels $h(X_i)$ for both
models. We initialize from $k_0,q_\ast,v_0$, with
$
S_0=\frac12(v_0k_0^\top+k_0v_0^\top)
$
having eigenvalues $1$ and $-0.05$, and train $k,q,v$ for $500$ steps
using Adam and squared loss. For softmax attention, the only change is
replacing
$\phi(A,V)(X)=VXX^\top AX$
by
$
VX\operatorname{softmax}(X^\top AX).
$

At training step $t$, let
$
S_t=\frac12(v_tk_t^\top+k_tv_t^\top).
$
Figure~\ref{fig:softmax-experiment} shows
$\sigma_2(S_t)/\sigma_1(S_t)$. For lightning attention, the ratio
rapidly converges to zero, consistent with
Theorem~\ref{thm:multi-token-rank-one}. For softmax,
the ratio decreases at a slower rate.

\begin{figure}[!h]
    \centering
    \includegraphics[width=\linewidth]{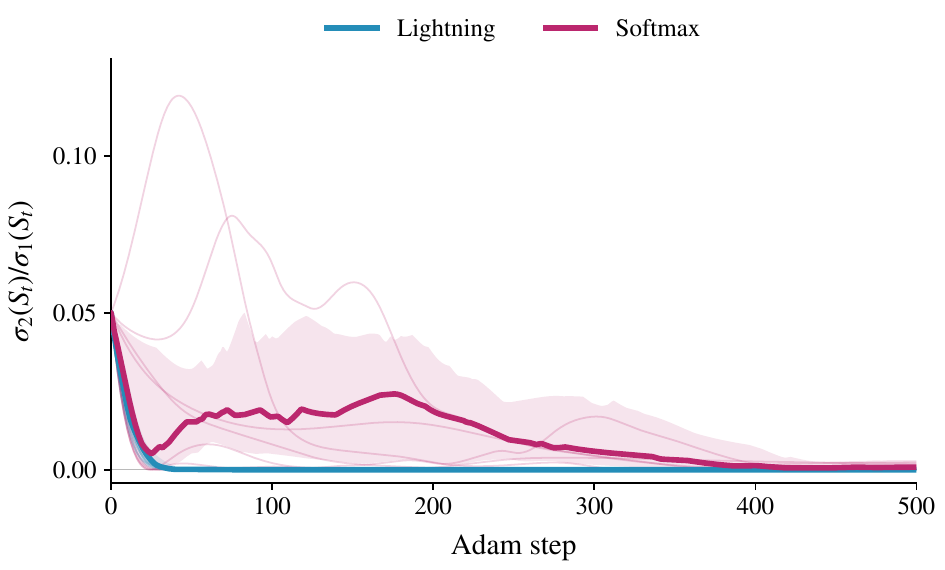}
    \caption{
    \textbf{Implicit bias for softmax attention.}
    Ratio
    $\sigma_2(S_t)/\sigma_1(S_t)$ over $500$ steps.
    Both models use the same inputs, labels, and initialization.
    Thick curves show medians over $30$ runs; shaded regions show
    interquartile ranges; faint curves show 8 individual~runs.}
    \label{fig:softmax-experiment}
\end{figure}

\section{Discussion}

Our results show that preference for submodels is not caused by singularities alone. In the models studied here, most lower-rank or inactive-neuron submodels are not preferred, even when they lie in the singular locus of the neural model. The exceptional cases reflect the real geometry, such as one-sided boundary strata or large normal cones.
In lightning self-attention, proximity bias arises at one-sided boundaries coming from a real inertia constraint.

It is natural to generalize these results to more expressive models. Low-rank matrix approximation is the starting point, from which one can move toward neural models along two natural hierarchies:
\[
\begin{aligned}
\text{matrix rank}
&\longrightarrow \text{tensor rank}\\
&\longrightarrow \text{shallow MLPs}\\
&\longrightarrow \text{deep MLPs}
\longrightarrow \cdots,
\end{aligned}
\]
\[
\begin{aligned}
\text{matrix rank}
&\longrightarrow \text{lightning self-attention}\\
&\longrightarrow \text{self-attention}
\longrightarrow \cdots.
\end{aligned}
\]
This paper takes steps along both hierarchies: from tensor rank and shallow MLPs toward deep networks, and from low-rank matrix approximation to lightning self-attention. The next steps are to understand preference for submodels in deep polynomial MLPs with arbitrary activation degrees when a neuron from any hidden layer is inactive (cf. Conjecture~\ref{conj:mlp-small-dim-no-bias}) 
and softmax self-attention. More broadly, when the geometry itself does not prefer a simpler submodel, an important question is whether preferences observed in practice come from the optimization dynamics or from structure in the~target.

\subsection*{AI use statement}

In this work, we used generative AI tools for writing code to generate figures using Mathematica, reviewing our proofs, filling in technical details of proofs, and polishing writing.
We have not used generative AI tools to come up with the main ideas and solve the main problems in this paper. We have reviewed all AI-assisted work. We take responsibility for the final content of this work, including text, claims, or artifacts produced with the aid of generative AI.

\subsubsection*{Acknowledgements}
We thank Thomas Grill, Kathl\'en Kohn, and Vahid Shahverdi for helpful discussions.
AR was supported by fellowships from “la Caixa” Foundation (ID 100010434), under the agreement LCF/BQ/EU23/12010097, from Rafael del Pino Foundation, and from Real Colegio Complutense at Harvard University. AS was supported by NSF DMS 2608217 and an Alfred P. Sloan research fellowship. We acknowledge support from an AMS–Simons Travel Grant (awarded to YA), which allowed AR to visit YA at UCLA to make progress on this project.

\bibliography{references}

\clearpage
\appendix
\thispagestyle{plain} 

\onecolumn
\aistatstitle{Supplementary Materials}

\section{Proofs of Propositions~{\ref{prop:implicit-bias-criterion} and \ref{prop:equal-dimension-bias}}} \label{sec:proof-toolbox}

\implicitbiascriterion*

\begin{lemma} \label{lem:local-minima-normal}
        Let $\ell : \cH \to \RR$ be Fr\'echet differentiable.
        If $f \in \cF$ is a local minimizer of $\ell\rvert_\cF$
        then $-\nabla \ell(f) \in N_\cF(f)$.
    \end{lemma}

    \begin{proof}
    Let $v\in T_\cF(f)$ and pick sequences $(f_n)_n \subset \cF$ and $(\tau_n)_n \subset \RR_{> 0}$
    such that $(f_n-f)/\tau_n\to v$. Since~$f$ is a local minimizer,
    $\ell(f_n)-\ell(f)\geq0$ for all sufficiently large $n$.
    By Fr\'echet differentiability of $\ell$,
    $$
    \ell(f_n)-\ell(f)
    =
    \langle\nabla \ell(f),f_n-f\rangle
    +o(\|f_n-f\|).
    $$
    Dividing by $\tau_n$ and taking the limit gives
    $\langle\nabla \ell(f),v\rangle\geq0$.
    \end{proof}

    \begin{proof}[Proof of \Cref{prop:implicit-bias-criterion}]
    We always have $f \in \Vor_\cF(f)$ because $\ell_f(f) = 0$ and $\ell_f\geq 0$.
If $h\in\operatorname{Vor}_{\cF}(f)$, then $f$ is a global, and hence
local, minimizer of $\ell_h$. If $f$ is a local
minimizer, Lemma~\ref{lem:local-minima-normal} gives
$-\nabla \ell_h(f) = 2(h-f)\in N_\cF(f)$. Since~$N_\cF(f)$ is a cone,
$h\in f+N_\cF(f)$.
\end{proof}

\begin{corollary}\label{cor:zero}
If $N_\cF(f)=\{0\}$ for every $f\in\cF'$, then $\Vor_\cF(\cF') =  \Vorloc_\cF(\cF')=\cF'$.
\end{corollary}

\equaldimensionbias*

\begin{proof}[Proof of \Cref{prop:equal-dimension-bias}]
    For the first claim, we use the following facts about semialgebraic sets $A \subseteq \RR^n$: $\dim (\overline A) = \dim (A)$ and $\dim (\overline{A} \setminus A) < \dim A$ for $A \neq \varnothing$ \citep[\S 2.8]{bochnak2013real}. Given $f \in \cF \subset \cH$ and $\varepsilon > 0$, denote the ball centered at $f$ of radius $\varepsilon$ by $B(f, \varepsilon) = \{h \in \cH : \| h - f\| \leq \varepsilon \}$.
    We claim that there exists $f \in \cF'$ with a neighborhood  in the closure $\overline{\cF}$ that is contained in $\cF'$.
    Otherwise every point of $\cF'$ would be a limit of points of $A = \overline{\cF} \setminus \cF'$, which implies that $\cF' \subseteq \overline{A} \setminus A$ and so $\dim \cF' < \dim A \leq \dim \cF$, a contradiction. 
    Assume $\varepsilon$ is such that $\overline{\cF} \cap B(f, 3 \varepsilon) \subseteq \cF'$.
    Next, we show that $\Vorloc_\cF(\cF') \supseteq B(f, \varepsilon)$. Pick any $h \in B(f, \varepsilon)$. Then, $\overline{\cF}$ has a point $\tilde f$ closest to $h$, and $\| h - \tilde f\| \leq \| h - f\| \leq \varepsilon$, so $\| \tilde f - f\| \leq 2\varepsilon$ and $\tilde f \in \cF'$. Hence, $\tilde f$ is a global minimizer of $\ell_h \rvert_\cF$ and $h \in \Vor_\cF(\cF') \subseteq \Vorloc_\cF(\cF')$, so $\Vorloc_\cF(\cF') \supseteq B(f, \varepsilon)$.
    
    For the second claim, assume $\dim \cH < \infty$ and $\dim \cF' < \dim \cF$. At a smooth point $f \in \cF$, its local Voronoi cell lies in the translated normal space $f + N_\cF(f)$, of dimension $\dim \cH - \dim \cF$. So over a submodel $\cF'$ of smooth points the (local) Voronoi cells fill a set of dimension at most $\dim \cF' + \dim \cH - \dim \cF < \dim \cH$.
\end{proof}

\section{Proofs of Multilayer Perceptron Results} \label{sec:proofs-mlp}

\subsection{Shallow MLPs}

\nobiasshallow*

A \emph{cone} in a real Hilbert space $\cH$ is a set $\cC \subseteq \cH$ that is closed under multiplication by nonnegative scalars, that is,  $f \in \cC, \lambda \geq 0 \implies \lambda f \in \cC$.
A cone $\cC$ is \emph{symmetric} if $\cC = - \cC$. Equivalently, a symmetric cone
is a set $\cC \subseteq \cH$ that is closed under multiplication by scalars: $f \in \cC, \lambda \in \RR \implies \lambda f \in \cC.$

    \begin{lemma}\label{lem:sym-cone-dense}
        If $\cC \subseteq\cH$ is a symmetric cone whose linear span is dense in $\cH$, then its polar cone is zero.
    \end{lemma}

    \begin{proof}
        The polar cone of $\cC$
        is equal to its orthogonal complement, because $\cC$ is symmetric. The orthogonal complement of~$\cC$ equals the orthogonal complement of its linear span, so it is zero.
    \end{proof}

    \begin{lemma} \label{lem:normal-cone-secants}
        
        Let $\cC$ be a symmetric cone whose linear span is dense in $\cH$. If $\cF' + \cC \subseteq \cF$, then $N_{\cF}(f) = \{0\}$
        for all $f \in \cF'$, and
        $N_{\overline{\cF}}(f) = \{0\}$ for all $f \in \overline{\cF'}$. Therefore, 
        \[
        \Vor_\cF(\cF') = \Vorloc_\cF(\cF') = \cF' \quad \text{and} \quad \Vor_{\overline{\cF}}\left( \overline{\cF'} \right) = \Vorloc_{\overline{\cF}}\left(\overline{\cF'} \right) = \overline{\cF'}.
        \]
    \end{lemma}
    \begin{proof}
        It suffices to show that $T_{\cF}(f)$ contains $\cC$ for all $f \in \cF'$, by \Cref{lem:sym-cone-dense}. Given any $v \in \cC$, we have $f_n = f + \frac{1}{n}v \in \cF$ for all $n \in \mathbb{N}$, so $v = \lim_{n \to \infty} (f_n - f)/\frac{1}{n} \in T_{\cF}(f)$. Similarly, if $f \in \overline{\cF'}$ and $v \in \cC$, then $f_n = f + \frac{1}{n} v \in \overline{\cF}$, so the same argument applies. The claims about (local) Voronoi cells follow from \Cref{prop:implicit-bias-criterion}.
    \end{proof}

    \begin{proof}[Proof of \Cref{thm:no-bias-shallow}]
        Consider the single-neuron cone 
        \[\cC = \left\{f(x) = \rho(\langle w_1,x\rangle + \beta)  w_2 + b :
    w_1 \in \RR^{d_0}, \beta \in \RR, 
     w_2, b \in \RR^{d_2}
 \right\}.\]
        Then $\cF' = \sum_{i = 1}^{d_1-1} \cC$ and $\cF = \cF' + \cC$.
        The linear span of $\cC$ is dense in~$\cH$, because
        $\linspan(\cC) = \linspan(\cF)$. Therefore, the statement follows from \Cref{lem:normal-cone-secants} and \Cref{cor:zero}.
    \end{proof}

\subsection{Deep MLPs}

Neural models of shallow MLPs are sums of cones. This no longer holds for deep MLPs. However, in some cases, we can still find a cone $\cC$ whose span is dense in $\cH$ and such that $\cF' + \cC \subseteq \cF$. In such cases, \Cref{lem:normal-cone-secants} implies that $\cF' = \Vor_\cF(\cF') = \Vorloc_\cF(\cF')$. Note that unlike in the shallow case we may not have the equality~$\cF' + \cC = \cF$.

We define the features of layer $j$ recursively by $h^{(0)}_\theta(x)=x$ and $h^{(j)}_\theta=\rho_j(W_jh^{(j-1)}_\theta+b_j)$; then $f_\theta=W_Lh^{(L-1)}_\theta+b_L$, and we write $h_\theta=h^{(L-2)}_\theta$ for the features that feed the last hidden layer, as in \Cref{sec:deep}.
We consider two scenarios, as in \Cref{sec:mlp}: (1) all activations are polynomials and $\cH=\linspan\cF$ is finite-dimensional with any inner product, (2) the activations are continuous, $\mu$ is a finite Borel measure with compact support $K$, and $\cH=\overline{\linspan\cF}\subseteq L^2(\mu;\RR^{d_L})$. 
The loss that we use is the squared distance $\ell_h(f)=\|h-f\|^2$. We write $\cP_{\le n} = \RR[x_1, \dots, x_{d_0}]_{\leq n}$ for the polynomials of degree at most $n$ in $d_0$ variables, 
so $\dim\cP_{\le n}=\binom{d_0+n}{d_0}$.

\inactivepaths*

\begin{lemma}\label{lem:cone-paths}
    Let $\mathcal{F}$ be the neural model of an MLP with architecture $\mathbf d = (d_0, d_1, \dots, d_{L-1}, d_L)$, let $\mathcal{F}'$ be the submodel with architecture $(d_0, d_1-1, \dots, d_{L-1}-1, d_L)$, and let $\tilde \cC$ be the submodel with architecture $(d_0, 1, \dots, 1, d_L)$. Then $\tilde \cC$ is a symmetric cone and $\cF' + \tilde \cC \subseteq \cF$.
\end{lemma}
\begin{proof}
Given parameters $\theta'$ of $\cF'$ and $\tilde \theta$ of $\tilde \cC$, define a parameter $\theta$ of $\cF$ with block-diagonal
weight matrices $W_j=\operatorname{diag}(W'_j,\tilde W_j)$ and biases $b_j=(b'_j,\tilde b_j)$, except that $W_1^\top = ({W'_1}^\top \  {\tilde W_1}^\top)$, $W_L=(W'_L \ \tilde W_L)$ and $b_L=b'_L+\tilde b_L$. Since activations act coordinatewise, $f_\theta=f_{\theta'}+f_{\tilde \theta}$. 
Scaling $(\tilde W_L,\tilde b_L)$ scales $f_{\tilde \theta}$, so $\tilde \cC$ is a
symmetric cone.
\end{proof}

To prove \Cref{thm:inactive-paths} it suffices to show that the span of the symmetric cone $\tilde \cC$ of the architecture $(d_0, 1, \dots, 1, d_L)$ is dense in $\cH$, by \Cref{lem:cone-paths} and \Cref{lem:normal-cone-secants}.
We do this separately for the two cases: the polynomial (finite-dimensional) case, and the case where not all activations are polynomial (which lives in an infinite dimensional function space).

\begin{lemma}\label{lem:ridge}
If $\sigma$ is a polynomial of degree $R\ge1$, then $\linspan\{\sigma(q):q\in\cP_{\le n}\}=\cP_{\le Rn}$.
\end{lemma}

\begin{proof}
Let $V=\linspan\{\sigma(q):q\in\cP_{\le n}\}$. Every polynomial in $V$ has degree at most $Rn$, hence $V\subseteq\cP_{\le Rn}$. Conversely, every monomial of degree at most $Rn$ is a product $m_1\cdots m_k$ of $k\le R$ monomials of degree at most $n$. For $t\in\RR^k$ and $s\in\RR$, the polynomial $q=t_1m_1+\dots+t_km_k+s$ lies in $\cP_{\le n}$, so $\sigma(q)\in V$. Expanding $\sigma(q)$ as a polynomial in $(t,s)$, only the top-degree term $\sigma_Rq^R$ of $\sigma$ contributes to the monomial $t_1\cdots t_k\,s^{R-k}$, and its coefficient is $\sigma_R\frac{R!}{(R-k)!}\,m_1\cdots m_k$ with $\sigma_R\neq0$. Since every coefficient of a polynomial map $(t,s)\mapsto\sigma(q)$ with values in $V$ is a linear combination of finitely many of its values, it lies in $V$. Hence $m_1\cdots m_k\in V$.
\end{proof}

In what follows, we use the universal approximation theorem \citep{leshno1993multilayer}, and abbreviate it to UAT: if $\sigma : \RR \to \RR$ is continuous and not a polynomial, then for every compact set $Y \subset \RR^{m}$ the functions $y \mapsto \sigma(\langle a, y\rangle + b)$ with $a \in \RR^m$ and $b \in \RR$ span a dense subset of the continuous functions $C(Y)$. In our case, since $\mu$ is finite with compact support $K$, $C(K)$ is dense in $L^2(K)$ and uniform convergence on $K$ implies convergence in $L^2(K)$.
The next lemma extends UAT to paths of single neurons, one per hidden layer.

\begin{lemma}\label{lem:paths}
Let activation functions $\rho_1,\dots,\rho_{L-1}\colon\RR\to\RR$ be continuous and nonconstant, not all polynomials, and 
let $K\subseteq\RR^{d_0}$ be compact. 
Let $\Pi_i$ be the set of functions computed by the neuron in
layer $i$ of a network with architecture $(d_0,1,\dots,1)$, that is, $\Pi_1=\{\rho_1(\langle a,x\rangle+b)\}$ and
$\Pi_i=\{\rho_i(wp+c):p\in\Pi_{i-1},\,w,c\in\RR\}$ for $1 < i < L$. Then $\linspan\Pi_{L-1}$ is dense in $C(K)$.
\end{lemma}

\begin{proof}
Let $V_i$ be the uniform closure of $\linspan\Pi_i$ in $C(K)$. If $p\in\Pi_{i-1}$ and $\varphi$ is a uniform limit on
compact intervals of linear combinations of $s\mapsto\rho_i(ws+c)$, then $\varphi\circ p\in V_i$, by substituting
$s=p(x)$. The identity $\varphi(s)=s$ is such a limit: if $\rho_i$ is a polynomial of degree $R\ge1$, the functions
$\rho_i(ws+c)$ span all polynomials of degree at most $R$ in $s$ by \Cref{lem:ridge}; otherwise apply UAT with $m=1$.
Hence $V_1\subseteq\dots\subseteq V_{L-1}$.
Now let $\rho_i$ be non-polynomial. Choose $w_j,c_j$ so that
$\tau=\rho_{i-1}(w_{i-1}\rho_{i-2}(\cdots\rho_1(\cdot)\cdots)+c_{i-1})$ is nonconstant ($\tau=\mathrm{id}$ if $i=1$). By UAT with input dimension $m=1$, $\varphi=\arctan$ is a uniform limit on compact intervals of combinations of
$\rho_i(ws+c)$, so $\psi(\langle a,x\rangle+b)\in V_i$ for all $a,b$, where $\psi=\arctan\circ\tau$. Since $\psi$ is
bounded and nonconstant, it is not a polynomial, and UAT gives $V_i=C(K)$. Therefore $V_{L-1} =  V_i=C(K)$.
\end{proof}

\begin{proof}[Proof of \Cref{thm:inactive-paths}]
It suffices to show that $\linspan \tilde \cC$ is dense in $\cH$, 
by \Cref{lem:cone-paths} and \Cref{lem:normal-cone-secants}. If all activations are polynomials, this follows from \Cref{lem:ridge} applied to $\sigma=\rho_{L-1}\circ\dots\circ\rho_1$, whose degree $r_1\cdots r_{L-1}$ bounds the degree of every function in $\cF$. Otherwise, it follows from \Cref{lem:paths}.
\end{proof}

\completefeaturesorthogonal*

Given a network parameter $\theta$, recall that $\cC_\theta = \{\rho_{L-1}(\langle w, h_\theta\rangle + \beta) v  : w \in \RR^{d_{L-2}}, \beta \in \RR, v \in \RR^{d_L}  \}$ is the symmetric cone consisting of functions computed by a single neuron of the last hidden layer with input features $h_\theta$. We say that $\theta$ has \emph{complete features}  if $\overline{ \linspan \cC_\theta} = \cH$. 
Now we can apply a point-wise analog of \Cref{lem:normal-cone-secants}.

\begin{proof}[Proof of \Cref{thm:complete-features-orthogonal}]
We have $f_\theta + \cC_\theta \subseteq \cF$ for each representation $\theta$ of the subnetwork. Since $\cC_\theta$ depends on $\theta$, we cannot write $\cF' + \cC_\theta \subseteq \cF$ and apply \Cref{lem:normal-cone-secants} directly. However, the proof of \Cref{lem:normal-cone-secants} is pointwise, so the same argument shows that $N_\cF(f_\theta) = \{0 \}$ and $\Vor_\cF(f_\theta) =  \Vorloc_\cF(f_\theta) = \{f_\theta\}$ if $f_\theta$ has complete features. Since every $f \in \cF'$ has a representation with complete features, the claim follows.
\end{proof}

\mlplastlayerpolynomials*

\begin{proof}[Proof of \Cref{cor:mlp-last-layer-polynomials}]
It suffices to find, for every $f\in\cF'$, a representation $\theta$ such that
$\cC_\theta$ spans the ambient space $\cH = \linspan \cF \subseteq \RR^{d_L} \otimes \cP_{\leq r_1 \cdots r_{L-1}}$,
by \Cref{thm:complete-features-orthogonal}.
The width condition allows us to choose the first $L-2$
layers so that $\linspan\{1,h^{(j)}_1,\dots,h^{(j)}_{d_j}\}=\cP_{\le r_1 \cdots r_j}$ for all $j\le L-2$,
by \Cref{lem:ridge} and induction on $j$.
In any representation of $f$, the pre-activations of layer $L-1$ lie in
$\cP_{\le r_1 \cdots r_{L-2}}=\linspan\{1,h^{(L-2)}_1, \dots, h^{(L-2)}_{d_{L-2}}\}$, so they can be written as $\tilde Wh^{(L-2)}+\tilde b$. Then $\linspan\cC_\theta=  \cH$,
by \Cref{lem:ridge}.
\end{proof}

\mlpfirstlayerpoly*

When all activation functions are polynomial, $\cH = \linspan \cF$ is finite dimensional.

\begin{lemma}\label{lem:curve}
If $\gamma: (-\varepsilon,\varepsilon)\to\cF \subseteq \cH$ with $\gamma(0)=f$ and $\gamma$ is differentiable at $0$, then $N_\cF(f) \perp \gamma'(0)$. Thus if there exists $c \in \cH$ such that $f+tc\in\cF$ for all $t\in\RR$, then $N_\cF(f)\perp c$.
\end{lemma}
\begin{proof}
For each $n \in \mathbb N$, take $f_n=\gamma(\pm1/n)$ and $\tau_n=1/n$, so both $\gamma'(0)$ and $-\gamma'(0)$ lie in $T_\cF(f)$. Hence, every $w\in N_\cF(f)$ satisfies $\langle w,\pm\gamma'(0)\rangle\le0$.
\end{proof}

\begin{lemma}\label{lem:const}
If $\cF$ be the neural model of an MLP with biases, then $N_\cF(c)=\{0\}$ for every constant function~$c$.
\end{lemma}
\begin{proof}
The model $\cF$ is a symmetric cone that contains the constant functions, so $c + \cF = \cF$ for all constants $c$ and the claim follows from \Cref{lem:normal-cone-secants}.
\end{proof}

Next, we show that if the normal cone at each point of a semialgebraic set $S$ is cut down by at least $m$ independent linear constraints, then the union of translated normal cones over $S$ has codimension at least $m - \dim S$. When $m > \dim S$, this union has positive codimension, ruling out proximity bias toward $S$ by \Cref{prop:implicit-bias-criterion}.

\begin{lemma}\label{lem:count}
Let $\Theta\subseteq\RR^P$ be semialgebraic and let $\Phi,v_1,\dots,v_r\colon\Theta\to\cH$ be semialgebraic maps with $\Phi(\Theta)\subseteq\cF$ and $N_\cF(\Phi(\theta))\perp v_i(\theta)$ for all $\theta,i$. Let $n = \dim \cH$, let $S\subseteq\Phi(\Theta)$ be semialgebraic and assume that for every $f\in S$ the vectors $v_i(\theta)$ with $\Phi(\theta)=f$ span a space of dimension at least $m$. Then $\bigcup_{f\in S}(f+N_\cF(f))$ lies in a semialgebraic set of dimension at most $\dim S+n-m$.
\end{lemma}
\begin{proof}
The set $E=\{(f,\nu)\in S\times\cH:\ \forall\theta\,(\Phi(\theta)=f\Rightarrow\langle\nu,v_i(\theta)\rangle=0\ \forall i)\}$ is
semialgebraic, by the Tarski--Seidenberg Theorem. Its fiber over $f$ is a linear space of dimension at most $n-m$ containing $N_\cF(f)$, so $\dim E\le\dim S+n-m$, and $\{f+\nu:(f,\nu)\in E\}$ is its image under a linear map.
\end{proof}

\begin{proof}[Proof of \Cref{thm:mlp-first-layer-poly}] Let $n = \dim \cH$ and $m = \binom{d_0 + r_1}{d_0}$.
We show that the normal cone at every nonconstant~$f\in\mathcal F'$ lies in a subspace of codimension $m$, and use \Cref{prop:implicit-bias-criterion} to conclude that $\Vorloc_\cF(\cF')$ lies in a semialgebraic set of dimension less than $n$, so it has empty interior.

Let $\theta'$ be a parameter of the submodel $\cF'$ and define $f=f_{\theta'}$. Write $f=\Psi(z)$, where $z$ is the input of the second activation (so $z=W_2\rho_1(W_1x+b_1)+b_2$) and $\Psi$ the polynomial map computed by the subsequent layers. Let $J = D\Psi\circ z$. That is, $J(x)= D\Psi(z(x))$ is the $d_L \times d_2$ Jacobian of $\Psi$ evaluated at $z(x)$ to give a matrix of polynomials in~$x$. We write $e_l$ for the $l$-th standard basis vector of $\RR^{d_2}$, so $J e_l$ is the $l$-th column of $J$ and for a polynomial~$g$, $J g e_l$ denotes the vector of polynomials $g \cdot J e_l$.

Adding a neuron to the first hidden layer with incoming weights $(a,b)$ and outgoing weights $tw$ with nonzero $w$ gives $f_t=\Psi(z+tw\,\rho_1(\langle a,x\rangle+b))\in\cF$, a polynomial curve whose derivative at $t=0$ is
$Jw\,\rho_1(\langle a,x\rangle+b)$. 
The curve passes through $f$ at $t=0$, so \Cref{lem:curve} gives $N_\cF(f)\perp Jw\,\rho_1(\langle a,x\rangle+b)$ for all $w\in\mathbb R^{d_2}$, $a\in\mathbb R^{d_0}$, $b\in\mathbb R$. Taking $w=e_l$ for $l=1,\dots,d_2$ and linear combinations over $(a,b)$ gives $N_\cF(f)\perp Jg\,e_l$ for all $l$ and all $g\in\cP_{\le r_1}$, by \Cref{lem:ridge}.

These directions give $m$ independent constraints whenever $f$ is not constant, as follows. Let $p_1, \dots, p_m$ be a basis of $\cP_{\leq r_1}$.
Since $f$ is nonconstant, the chain rule gives $0\ne Df= J Dz= \sum_l(Je_l)(\nabla z_l)^\top$, 
so some column $Je_l$ is a nonzero vector of polynomials.
Hence, $g\mapsto Jg e_l$ is injective and so the $m$ functions $J p_1e_l , \dots, J p_m e_l$ are linearly independent.

Let $\Theta$ 
be the parameter space of $\cF'$, and apply Lemma~\ref{lem:count} to $v_{l,j}(\theta')=J_{\theta'}p_j e_l$, and the set $S$ of nonconstant elements of $\cF'$: the union of $f+N_\cF(f)$ over $f\in S$ lies in a semialgebraic set of dimension $\le\dim\cF'+n-m<n$.
Constant $f$ contribute only $f$ itself (Lemma~\ref{lem:const}), a set of dimension at most $\dim\cF'<m\le n$; here $m\le n$ because $S\ne\emptyset$ (as $d_1-1\ge1$) and the previous paragraph gives $m$ independent vectors in $\cH$.
\end{proof}

The upper bound in \Cref{thm:mlp-first-layer-poly} can be improved for polynomial neural networks with no biases $(b_i = 0)$. In this case, the dimension of \emph{homogeneous} forms of degree $r_1$ is $\dim\RR[x_1,\dots,x_{d_0}]_{r_1} = \binom{d_0+r_1-1}{r_1}$.

\begin{proposition}\label{prop:mlp-first-layer-mono}
Let $\cF$ be the neural model of a deep biasless MLP with monomial activations
$\rho_j(t)=t^{r_j}$, $r_j\ge1$, and let $\cF'$ be the submodel with width $d_1-1$ in the first hidden layer. If
$
\dim\cF'<\binom{d_0+r_1-1}{r_1},
$
then there is no proximity bias toward $\cF'$.
\end{proposition}

\begin{proof}
The proof of \Cref{thm:mlp-first-layer-poly} applies with $m=\binom{d_0+r_1-1}{r_1}$ and the following two changes. First, the reinserted neuron now computes $\langle a,x\rangle^{r_1}$, and these powers span the space $\RR[x_1,\dots,x_{d_0}]_{r_1}$ of forms of degree $r_1$. Hence $N_\cF(f)\perp Jge_l$ for all $g\in\RR[x_1,\dots,x_{d_0}]_{r_1}$. 
The only constant function in $\cF$ is $0$, because every $f\in\cF$ is homogeneous of degree $r_1\cdots r_{L-1}\ge1$. 
Let $p_1, \dots, p_m$ be a basis of $\RR[x_1, \dots, x_{d_0}]_{r_1}$.
For nonconstant $f$,
the $m$ functions $Jp_1 e_l, \dots, J p_m e_l$ are linearly independent by the same injectivity argument as the one in \Cref{thm:mlp-first-layer-poly}.
Since scaling $W_L$ gives $t\cF\subseteq\cF$ for all $t\in\RR$, \Cref{lem:curve} gives $N_\cF(0)\perp\linspan\cF=\cH$, so $N_\cF(0)=\{0\}$.
\end{proof}

The following example shows that the bound in \Cref{prop:mlp-first-layer-mono} is sharp.

\begin{example}[Rank two versus rank three]\label{ex:rank-two-three}
Let $\mathbf d=(2,3,1,1)$ and $\rho_1=\rho_2: t \mapsto t^3$, and no bias terms. Every real binary cubic is a sum of three cubes of real linear forms, and it is a sum of two cubes if and only if it is a cube or has negative discriminant (one real and two complex conjugate roots); see, e.g., \cite{comon2012typical}. Hence $\mathcal F=\{p^3 : p \text{ a binary cubic}\}$ and $\mathcal F'=\{p^3 : \operatorname{disc}(p)<0\}\cup\{\ell^9\}$: the model consists of the cubes of the cubics of real rank at most three, the submodel of those of real rank at most two. Since $p\mapsto p^3$ is injective, $\dim\mathcal F'=\dim\mathcal F=4$, and $\mathcal F'\neq\mathcal F$. By \Cref{prop:equal-dimension-bias} the local Voronoi cell $\Vorloc_{\mathcal F}(\mathcal F')$ has nonempty interior: every target close to $p_0^3$ with $\operatorname{disc}(p_0)<0$ has all its nearest points of $\mathcal F$ in $\mathcal F'$. Here $\dim\mathcal F'=4=\binom{2+3-1}{3}$, so the bound in \Cref{prop:mlp-first-layer-mono} cannot be relaxed.
\end{example}

\anylayerpoly*
\begin{proof}
    The argument is the same as the one in \Cref{thm:mlp-first-layer-poly}, but applied at layer $k$ instead of layer 1. 
    Let $R = \deg \rho_k$. We show that the functions computed by a single neuron in layer $k$ span all polynomials of degree at most $R$ in the features. This gives $R+1$ linear constraints on the normal cones $N_\cF(f)$ for nonconstant $f \in \cF'$. We then apply the dimension count from \Cref{lem:count}.

 For a parameter $\theta'$ of $\cF'$, let $\eta = h^{(k-1)}_{\theta'}$ be the features that are the input to layer $k$, and write $f = f_{\theta'} = \Psi \circ z$ where $z = W_{k+1} \tilde \rho_k (W_k \eta + b_k) + b_{k+1}$ and $\Psi$ is the map computed by the remaining layers. Let $J = D \Psi \circ z$. As in the proof of \Cref{thm:mlp-first-layer-poly}, with the features $\eta$ in place of the input variables $x$, the inactive neuron gives $N_\cF(f) \perp J e_l g(\eta)$ for all $l$ and all polynomials $g$ of degree at most $R$. If $f$ is nonconstant, some $J e_l$ is nonzero and some coordinate $\eta_i$ is nonconstant. So the $R+1$ functions $J e_l \eta_i^{j}$ with $0 \leq j \leq R$ are linearly independent. 
 By \Cref{lem:count} with $m = R+1$, applied to $v_{j}(\theta') = J e_l \eta_i^j$ with $0 \leq j \leq R$ and the set $S$ of nonconstant functions in $\cF'$, the union of $f + N_\cF(f)$ over $f \in S$ has codimension at least $R + 1 - \dim \cF' > 0$. Constant $f$ contribute only $f$ itself (\Cref{lem:const}), so $\Vorloc_{\cF}(\cF')$ has empty interior by \Cref{prop:implicit-bias-criterion}.
\end{proof}

\mainnonpolyanyhiddenlayer*

Notation: $\theta'$ is a parameter of the submodel, $\eta=h^{(k-1)}_{\theta'}$ the features that are the input to layer $k$
($\eta(x)=x$ if $k=1$), $z$ the pre-activations of layer $k+1$, $\Psi$ the map computed by the layers above, so that
$f_{\theta'}=\Psi\circ z$, and $J=D\Psi\circ z$ ($J=I$ if $k=L-1$). All these maps are real-analytic in $(\theta',x)$.

\begin{lemma}\label{lem:Jdir}
For every continuous $\gamma\colon\RR^{d_{k-1}}\to\RR$ and every $l$: $Je_l\,(\gamma\circ\eta)\in\cH$ and
$N_\cF(f_{\theta'})\perp Je_l\,(\gamma\circ\eta)$.
\end{lemma}
\begin{proof}
Adding back the inactive neuron with incoming weights $(a,b)$ and outgoing weights $te_l$ gives the curve $f_t=\Psi(z+te_l\phi)\in\cF$, $\phi=\rho_k(\langle a,\eta\rangle+b)$. Since $D\Psi$ is uniformly continuous on compact sets, $(f_t-f)/t\to Je_l\phi$ uniformly on $K$. Hence in $\cH$, and $N_\cF(f)\perp Je_l\phi$ by Lemma~\ref{lem:curve}. Then, $\gamma\circ\eta$ is a uniform limit on $K$ of linear combinations of such $\phi$, by UAT applied on the compact set $\eta(K)$ since
$J$ is bounded on $K$.
\end{proof}

\begin{lemma}\label{lem:analytic}
Let $\mu$ have a Lebesgue density. (i) A real-analytic function on $\RR^{d_0}$ that is not identically zero vanishes only on a set of Lebesgue measure zero.
(ii) If $f_{\theta'}$ is not constant, there are $l$ and a nonconstant coordinate $\eta_i$ of $\eta$ such that $Je_l\ne0$ $\mu$-almost everywhere, 
and for every $N$ the functions $Je_l\,\eta_i^{\,j}$, $0\le j\le N$, are linearly independent in $\cH$.
\end{lemma}
\begin{proof}
(i) Induction on $d_0$ with Fubini: for $d_0=1$ the zeros are isolated; in general, the set of $x'$ for which $x_{d_0}\mapsto g(x',x_{d_0})$ vanishes identically is the common zero set of the analytic functions $x'\mapsto\partial^m_{x_{d_0}}g(x',0)$, which is null unless all of them vanish, i.e.\ unless $g=0$. (ii) $Df_{\theta'}=J\,Dz\ne0$, so some entry of some column $Je_l$ is a nonzero analytic function, nonzero almost everywhere\ by (i).
Since $f_{\theta'}$ is a function of $\eta$, some $\eta_i$ is nonconstant. If $Je_l\,p(\eta_i)=0$ almost everywhere for a polynomial $p$, then $p(\eta_i)=0$ almost everywhere. Therefore, (i) implies that $p=0$ on the interval $\eta_i(\RR^{d_0})$, so $p=0$.
\end{proof}

\begin{lemma}\label{lem:sweep}
Let $V$ be a finite-dimensional Euclidean space, $\Theta\subseteq\RR^P$ compact, and $A\colon\Theta\to\operatorname{Hom}(V,\RR^M)$, $c\colon\Theta\to\RR^M$ Lipschitz with $\operatorname{rank}A(\theta)\ge P+1$ for all $\theta$. Then $S=\{v\in V:\ A(\theta)v=c(\theta)\text{ for some }\theta\in\Theta\}$ has Lebesgue measure zero.
\end{lemma}
\begin{proof}
By compactness and continuity there are $\sigma_0>0$ and, for each $\theta$, a subspace $W_\theta$ of dimension $P+1$ with $|A(\theta)w|\ge\sigma_0|w|$ on $W_\theta$. Fix $R$ and cover $\Theta$ by $M_\varepsilon\le C_1\varepsilon^{-P}$ balls of radius $\varepsilon$ centered at $\theta_i$.
If $|v|\le R$ and $A(\theta)v=c(\theta)$ with $|\theta-\theta_i|\le\varepsilon$, then $|A(\theta_i)v-c(\theta_i)|\le C_2\varepsilon$, so $v$ lies in a set whose sections parallel to $W_{\theta_i}$ have diameter $\le 2C_2\varepsilon/\sigma_0$; its volume is $\le C_3\varepsilon^{P+1}$ by Fubini.
Hence $\lambda(S\cap\{|v|\le R\})\le C_1C_3\varepsilon\to0$.
\end{proof}

\begin{proof}[Proof of \Cref{thm:main-nonpoly-any-hidden-layer}]
    We start with the case $k=1$, $d_L=1$, where the normal cones vanish. Let $f\in\cF'$ and $\nu\in N_\cF(f)$. If $f$ is
constant, $\nu=0$, by Lemma~\ref{lem:const}. Otherwise take $\theta'$ representing $f$ and $l$ as in
Lemma~\ref{lem:analytic}. Here, $\eta(x)=x$ and $Je_l$ is a bounded scalar function that is nonzero almost everywhere. Then $\nu\perp Je_l\,g$ for all $g\in L^2(\mu)$,
by
Lemma~\ref{lem:Jdir} and density of $C(K)$ in $L^2(\mu)$.
Taking
$g=Je_l\,\nu$ gives $\int(Je_l)^2\nu^2\,d\mu=0$, so $\nu=0$. Thus $N_\cF(f)=\{0\}$ on $\cF'$, and
$\Vorloc_\cF(\cF')=\cF'$, by Corollary~\ref{cor:zero}.

For general $k$ and $d_L$ the normal cones need not vanish, but they satisfy more independent linear constraints than $\cF'$ has parameters.
In finite dimension a dimension count would finish the proof; in $\cH$ we run the count on finite-dimensional slices and conclude with the Baire category theorem. 
Let $P'$ be the number of parameters of the subnetwork and, for $\alpha=(l,i,j)$ with $j\le P'$, let $\chi_\alpha(\theta')=J_{\theta'}e_l\,(\eta_{\theta',i})^j$. If $h\in\Vorloc_\cF(\cF')$ has the local minimizer $f=f_{\theta'}\in\cF'$, then $h-f\in N_\cF(f)\perp\chi_\alpha(\theta')$ for all $\alpha$,
by Lemma~\ref{lem:Jdir}. For $n\ge1$ let
\[
\Theta_n=\{\theta':|\theta'|\le n,\ \operatorname{dist}(f_{\theta'},\cC)\ge1/n\},\qquad B_n=\{h:\ \exists\theta'\in\Theta_n\ \ \langle h-f_{\theta'},\chi_\alpha(\theta')\rangle=0\ \forall\alpha\},
\]
where $\cC$ is the space of constant functions. A constant local minimizer equals $h$, so $\Vorloc_\cF(\cF')\subseteq\cC\cup\bigcup_nB_n$, by Lemma~\ref{lem:const}. The maps $\theta'\mapsto f_{\theta'}$ and $\theta'\mapsto\chi_\alpha(\theta')$ are Lipschitz from $\{|\theta'|\le n\}$ to $\cH$ (their $\theta'$-derivatives are bounded on $\{|\theta'|\le n\}\times K$). Hence $\Theta_n$ is compact and $B_n$ is closed.

Next we show that each $B_n$ is nowhere dense. For each $\vartheta\in\Theta_n$ there are $P'+1$ indices $\alpha_0,\dots,\alpha_{P'}$ with $\chi_{\alpha_r}(\vartheta)$ linearly independent,
by Lemma~\ref{lem:analytic}(ii).
The matrix $(\langle\chi_{\alpha_r}(\vartheta),\chi_{\alpha_s}(\theta')\rangle)_{r,s}$ is then invertible for $\theta'$ near $\vartheta$.
Covering the compact set $\Theta_n$ by finitely many such neighborhoods, the finite-dimensional space $V$ spanned by the corresponding $\chi_\alpha(\vartheta_m)$ has the property that $A_{\theta'}\colon v\mapsto (\langle v,\chi_\alpha(\theta')\rangle)_\alpha$ has rank $\ge P'+1$ on $V$ for all $\theta'\in\Theta_n$. For any $h_0\in\cH$, $h_0+v\in B_n$ iff $A_{\theta'}v=(\langle f_{\theta'}-h_0,\chi_\alpha(\theta')\rangle)_\alpha$ for some $\theta'\in\Theta_n$, so every slice $\{v\in V:h_0+v\in B_n\}$ has measure zero by Lemma~\ref{lem:sweep}. A closed set containing a ball would contain a ball of $h_0+V$ in some slice, so $B_n$ is nowhere dense.

Finally, $\cC$ is a proper closed subspace, hence nowhere dense. So $\Vorloc_\cF(\cF')$ is contained in a countable union of closed nowhere dense sets, and by the Baire category theorem it has empty interior.
\end{proof}

\section{Convolutional neural networks}\label{sec:cnn-proof}

\begin{figure}[t]
    \centering
    \includegraphics[width=0.25\linewidth]{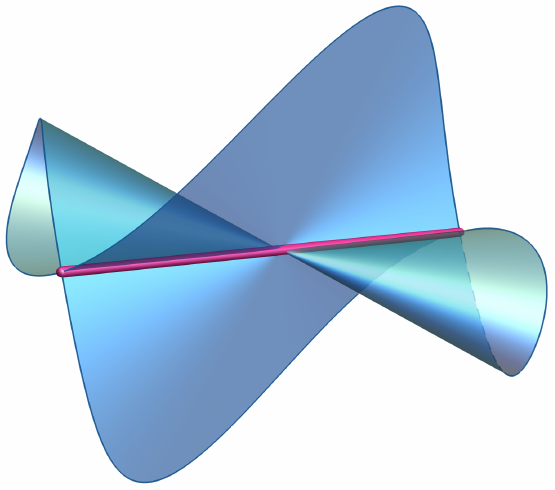}    
    \caption{
    \textbf{No proximity bias toward subnetworks of polynomial CNNs.}}
    \label{fig:cnn}
\end{figure}

A one-dimensional \emph{convolutional neural network (CNN)} with $L\ge2$ layers, filters $w_l\in\mathbb{R}^{k_l}$ and strides~$s_l$ is the map
\[
f_{\mathbf w}(x)=w_L\star_{s_L}\sigma\bigl(\cdots\,\sigma(w_1\star_{s_1}x)\bigr),
\]
where $(w\star_sh)[i]=\textstyle\sum_{j=1}^{k}w[j]\,h[s(i-1)+j]$, with one channel and no bias terms. We assume that the activation $\sigma$ is a generic polynomial of large enough degree depending on $L$ \citep{shahverdi2026learningrazorsedgeidentifiability}
such that $\sigma(0)=0$. Then the neural model $\mathcal F=\{f_{\mathbf w}\}$ is a semialgebraic subset of a finite-dimensional space $\mathcal H$ of polynomial maps, which we equip with any inner product. The natural submodels have smaller filters: for integers $0\le A_l<k_l$, not all zero, $\mathcal F_{\mathbf A}$ consists of the functions
$f_{\mathbf w}$ such that, for every $l$, the first or the last $A_l$ entries of $w_l$ vanish.

\begin{proposition}\label{prop:cnn}
    Let $\cF$ be the neural model of a CNN with a generic polynomial activation $\sigma$ of large enough degree such that $\sigma(0) = 0$, and let $\mathcal F'\subseteq\mathcal F$ be a semialgebraic subset with $\dim\mathcal F'<\dim\mathcal F$. Then, $\Vorloc_{\cF}(\cF')$ is contained in a semialgebraic set of dimension at most
$\dim\mathcal H-(\dim\mathcal F-\dim\mathcal F')$. Therefore, $\mathcal F$ exhibits no proximity bias toward any
submodel $\mathcal F_{\mathbf A}$ with smaller filters, nor toward its singular locus.
\end{proposition}

\begin{proof}
    Let $\varphi\colon\RR^m\to\cH$, $\bw\mapsto f_\bw$, be the parametrization and $n=\dim\cH$. By
\citet[Theorem~4.4]{shahverdi2026learningrazorsedgeidentifiability}, $d\varphi_\bw$ is injective whenever $f_\bw\neq0$, so $\dim\cF=m$. We have $N_\cF(f_\bw)\subseteq(\operatorname{im}d\varphi_\bw)^\perp$, a space of dimension $n-m$, for every
$f_\bw\in\cF'\setminus\{0\}$. Hence, $\bigcup_{f\in\cF'\setminus\{0\}}(f+N_\cF(f))$ lies in
a semialgebraic set of dimension at most $\dim\cF'+n-m$, by \Cref{lem:count}. At the origin, $t\cF\subseteq\cF$ for all $t\in\RR$, so $N_\cF(0)\subseteq(\linspan\cF)^\perp$, of dimension at most $n-m$. By \Cref{prop:implicit-bias-criterion},
$\Vorloc_\cF(\cF')$ lies in a semialgebraic set of dimension at most $n-(\dim\cF-\dim\cF')<n$.
Finally, $\cF_\bA$ is a finite union of images of coordinate subspaces of dimension $m-\sum_lA_l<m$, and the singular
locus of the Zariski closure $X$ of $\cF$ is a proper algebraic subset of the irreducible variety $X$, so both have
dimension less than $\dim\cF$. See \Cref{fig:cnn}.
\end{proof}

\section{Proofs of the Multiple-Token Results}
\label{app:multi-token-attention}

Throughout this section, assume that $t>1$ and $d>1$.

\subsection{Lower-rank submodels}

\multitokenlowerrank*

We first record the tangent directions available at a rank-deficient
attention matrix. Let
$D_a=\{A\in\RR^{d\times d}:\rank(A)\leq a\}$.

\begin{proposition}\label{prop:tangent-cone-attention-rank-deficient}
Suppose that $A\in D_a$ has rank $s<a$ and a decomposition
$A=\sum_{i=1}^s k_iq_i^\top$. Then
$$
T_{D_a}(A)=
\left\{
\sum_{i=1}^s
(\dot k_iq_i^\top+k_i\dot q_i^\top)
+\sum_{j=1}^{a-s}u_jv_j^\top \; : \; \dot k_i, \dot q_i , u_j, v_j \in \RR^d \right\}.
$$
\end{proposition}

\begin{proof}
Let $\mathcal R_s = \{A\in\RR^{d\times d}:\rank(A)=s\}$ denote the rank-$s$ stratum. 
By
\cite[Theorem~4.1]{olikier2026tangent},
$T_{D_a}(A)=T_{\mathcal R_s}(A)+D_{a-s}$.
Differentiating $A=\sum_{i=1}^s k_iq_i^\top$ gives
$T_{\mathcal R_s}(A)
=\{\sum_i(\dot k_iq_i^\top+k_i\dot q_i^\top)\}$.
Since every matrix in $D_{a-s}$ is a sum of at most $a-s$ rank-one
matrices, the result follows.
\end{proof}

The proposition describes the possible first-order variations of the
attention matrix $A$. We now use these variations to obtain tangent
directions to $\mathcal X_a$. Let $p=\phi(A,V)$, where
$\rank(A)=s<a$, and let $B\in T_{D_a}(A)$. Writing $B$ as in
Proposition~\ref{prop:tangent-cone-attention-rank-deficient}, define
$$
A(t)
=
\sum_{i=1}^s
(k_i+t\dot k_i)(q_i+t\dot q_i)^\top
+
t\sum_{j=1}^{a-s}u_jv_j^\top.
$$
Then $A(t)\in D_a$, $A(0)=A$, and $A'(0)=B$. For
$\dot V\in\RR^{d'\times d}$, set $V(t)=V+t\dot V$. Differentiating
$\phi(A(t),V(t))$ at $t=0$ gives
$
\frac{d}{dt}\phi(A(t),V(t))(X)\rvert_{t=0}
=
\dot VXX^\top AX+VXX^\top BX.
$
Therefore,
$$
T_{\mathcal X_a}(p)\supseteq
\left\{
\dot VXX^\top AX+VXX^\top BX:
\dot V\in\RR^{d'\times d},\ B\in T_{D_a}(A)
\right\}.
$$

For $p=\phi(A,V)$, define
$$
\mathcal J_{A,V}
=
\left\{
\dot VXX^\top AX+VXX^\top BX:
\dot V\in\RR^{d'\times d},\
B\in\RR^{d\times d}
\right\}.
$$

\begin{proposition}\label{prop:normal-cone-attention-rank<a}
Let $p=\phi(A,V)$, where $\rank(A)<a$. Then
$N_{\mathcal X_a}(p)\subseteq\mathcal J_{A,V}^{\perp}$.
\end{proposition}

\begin{proof}
Since $\rank(A)=s<a$, every rank-one matrix $B$ and its negative belong
to $D_{a-s}\subseteq T_{D_a}(A)$. Hence $T_{\mathcal X_a}(p)$ contains both
$VXX^\top BX$ and $-VXX^\top BX$ for every rank-one $B$.
Varying $V$ similarly gives both
$\dot VXX^\top AX$ and $-\dot VXX^\top AX$ for every $\dot V$.

If $w\in N_{\mathcal X_a}(p)$, every one of these directions is therefore
orthogonal to $w$. Since rank-one matrices span $\RR^{d\times d}$,
$w$ is orthogonal to all of $\mathcal J_{A,V}$.
\end{proof}

The dimension of the normal cone determines whether a lower-dimensional
stratum can exhibit proximity bias. Indeed, if $\mathcal S\subseteq
\mathcal X_a$ has dimension $r$ and
$\dim N_{\mathcal X_a}(p)\leq n-m$ for every $p\in\mathcal S$, where
$n$ is the ambient dimension, then
$$
\dim\left(
\bigcup_{p\in\mathcal S}
(p+N_{\mathcal X_a}(p))
\right)
\leq r+n-m.
$$
Thus, $\mathcal S$ cannot exhibit proximity bias whenever $r<m$.

If $p$ is a smooth point of $\mathcal X_a$, then
$\dim N_{\mathcal X_a}(p)=n-\dim\mathcal X_a$. Hence a
proper lower-dimensional submodel cannot exhibit proximity bias through
its smooth points. Only points that are singular in $\mathcal X_a$ can
have sufficiently enlarged normal cones to potentially compensate for
the lower dimension of the submodel. We show below that the exceptional
lower-rank strata still do not have sufficiently large normal cones. By Proposition~\ref{prop:normal-cone-attention-rank<a}, for a point $p=\phi(A,V)$ with $\rank(A)<a$, it suffices to lower bound
$\dim\mathcal J_{A,V}$.

\begin{lemma}\label{lem:attention-differential-t>1}
Suppose that $A\neq0$ and $V\neq0$. If there do not exist nonzero
vectors $k,q,h$ such that $A=kq^\top$ and $V=hk^\top$, then
$
\dim\mathcal J_{A,V}=d^2+dd'-1.
$
If $A=kq^\top$ and $V=hk^\top$, then
$
\dim\mathcal J_{A,V}=d^2+dd'-d.
$
\end{lemma}

\begin{proof}
Let $(B,\dot V)$ lie in the kernel of the differential. Since $t>1$,
comparing terms involving two distinct input columns $x$ and $y$ gives
$(\dot Vy)(y^\top Ax)+(Vy)(y^\top Bx)=0$ for all $x,y\in\RR^d$.
If $v_i^\top$ and $\dot v_i^\top$ are the
$i$th rows of $V$ and~$\dot V$, then
$(\dot v_i^\top y)(y^\top Ax)+(v_i^\top y)(y^\top Bx)=0$.

Suppose that, for some $i$ and $x$, the nonzero linear forms
$v_i^\top y$ and $y^\top Ax$ are not proportional. Since they are
relatively prime, the identity implies
$\dot v_i=-\lambda v_i$ and $Bx=\lambda Ax$ for some $\lambda\in\RR$.
Varying $i$ and $x$ then gives $\dot V=-\lambda V$ and $B=\lambda A$.
Thus the kernel is one-dimensional.

Otherwise, every nonzero row of $V$ is proportional to every nonzero
vector in $\operatorname{im}(A)$. Hence $A=kq^\top$ and $V=hk^\top$.
Substituting and cancelling $k^\top y$ gives
$(q^\top x)\dot Vy+h\,y^\top Bx=0$.
Taking $x\in\ker(q^\top)$ shows that $B=rq^\top$ for some $r\in\RR^d$,
and then $\dot V=-hr^\top$. Thus the kernel is $d$-dimensional.
\end{proof}

For $s<a$, let
$
\mathcal U_s\subseteq
\mathcal X_s\setminus\mathcal X_{s-1}
$
be the set of points $p$ admitting a rank-$s$ parametrization
$p=\phi(A,V)$ such that
$
\dim\mathcal J_{A,V}
=
d^2+dd'-1
=
\dim\mathcal X_d.
$

\begin{proposition}\label{prop:attention-low-rank-max-differential}
There is no proximity bias toward $\mathcal U_s$ in $\mathcal X_a$.
\end{proposition}

\begin{proof}
Let $n$ denote the ambient dimension. For every $p\in\mathcal U_s$,
Proposition~\ref{prop:normal-cone-attention-rank<a} gives
$
\dim N_{\mathcal X_a}(p)\leq n-\dim\mathcal X_d.
$
Therefore
\begin{align*}
\dim\left(
\bigcup_{p\in\mathcal U_s}
(p+N_{\mathcal X_a}(p))
\right)
&\leq \dim\mathcal U_s+n-\dim\mathcal X_d \leq \dim\mathcal X_s+n-\dim\mathcal X_d =n-(d-s)^2<n,
\end{align*}
where we used~\citep{HMK} to obtain $\dim\mathcal X_s=2sd+dd'-s^2-1$ and $\dim\mathcal X_d=d^2+dd'-1.$
Thus, there is no proximity bias toward $\mathcal U_s$ by \Cref{prop:implicit-bias-criterion}.
\end{proof}

For $s\geq2$, Lemma~\ref{lem:attention-differential-t>1} shows that
$\mathcal U_s=\mathcal X_s\setminus\mathcal X_{s-1}$.
For $s=1$, the only points not covered by
Proposition~\ref{prop:attention-low-rank-max-differential}
are contained in
\begin{equation}\label{eq:rank-one-attention}
    Y=\{\phi(kq^\top,hk^\top):
k,q\in\RR^d\setminus\{0\},\
h\in\RR^{d'}\setminus\{0\}\}.
\end{equation}

\begin{proposition}\label{prop:multi-token-no-bias-Y}
Suppose that $a>1$. There is no proximity bias toward $Y$ in $\mathcal X_a$.
\end{proposition}

\begin{proof}
We have $\dim Y\leq d'+2d-2$, by counting $k,q,h$ and subtracting the
two scaling symmetries. By
Lemma~\ref{lem:attention-differential-t>1},
$\dim\mathcal J_{A,V}=d^2+dd'-d$ at a parametrization
$A=kq^\top$, $V=hk^\top$. Hence
Proposition~\ref{prop:normal-cone-attention-rank<a} gives
$\dim N_{\mathcal X_a}(p)\leq n-(d^2+dd'-d)$ for $p\in Y$. Therefore
$$
\dim\left(\bigcup_{p\in Y}(p+N_{\mathcal X_a}(p))\right)
\leq n-(d-1)(d+d'-2)<n.
$$
Thus the union of the corresponding translated normal cones is not
full-dimensional, so there is no proximity bias toward $Y$.
\end{proof}

\begin{proof}[Proof of Theorem~\ref{thm:multi-token-lower-rank}]
Proposition~\ref{prop:attention-low-rank-max-differential} and
Proposition~\ref{prop:multi-token-no-bias-Y} show there is no
proximity bias toward any nonzero rank-exactly-$r$ stratum with $r<a$.
At the origin, since $\mathcal X_a$ is a cone and
$\mathcal X_a=-\mathcal X_a$, we have 
$
N_{\mathcal X_a}(0)=\operatorname{span}(\mathcal X_a)^\perp,
$
which is a proper linear subspace of the ambient space. Hence there is
no proximity bias toward the origin. Since each $\mathcal X_s$ with $s<a$
is the finite union of the origin and the rank-$r$ strata with $r\leq s$, there is no proximity bias toward $\mathcal X_s$.
\end{proof}

\subsection{Singular loci}

For the algebraic analysis in this subsection, we work over $\CC$. Let
$
D_a^\CC=\{A\in\CC^{d\times d}:\rank(A)\leq a\},
$
let $\mathcal X_a^\CC$ denote the complexification of $\mathcal X_a$,
and let $\PP(\mathcal X_a^\CC)$ denote its projectivization. We also set
$$
Z^\CC=
\{\phi(A,V):
A\in\CC^{d\times d},\
V\in\CC^{d'\times d},\
\rank(A)=\rank(V)=1\}.
$$

\begin{lemma}\label{lem:multi-token-singular}
Assume that either $a>1$ or $d'>1$. Then
$$
\operatorname{Sing}\bigl(\PP(\mathcal X_a^\CC)\bigr)
=
\begin{cases}
\PP(\mathcal X_{a-1}^\CC), & 1<a<d,\\
\PP(Z^\CC), & a=d\text{ or }(a=1,\ d'>1).
\end{cases}
$$
\end{lemma}

We first establish the finite-birational statement used in the proof.

\begin{proposition}\label{prop:multi-token-finite-birational}
The projectivized parametrization
$
\overline{\phi}:
\PP(D_a^\CC)\times\PP(\CC^{d'\times d})
\longrightarrow\PP(\mathcal X_a^\CC)
$
has finite fibers. It is birational when $a>1$, and also when
$a=1$ and $d'>1$. When $a=d'=1$, it is generically two-to-one.
\end{proposition}

\begin{proof}
Suppose that $\phi(A,V)=\phi(A',V')$. Since $t>1$, comparing terms
involving two distinct input columns gives
$
(v_i^\top y)(y^\top Ax)
=
(v_i'{}^\top y)(y^\top A'x)
$
for every row $i$ and all $x,y\in\CC^d$.
If $\rank(A)>1$, the bilinear form $y^\top Ax$ is irreducible.
Unique factorization therefore gives $A'=\lambda A$ and
$V'=\lambda^{-1}V$, so the projective fiber is a singleton.

Suppose now that $\rank(A)=1$ and write $A=kq^\top$. Unique
factorization first gives $A'=k'q^\top$, up to scaling, and then
$
(v_i^\top y)(k^\top y)
=
(v_i'{}^\top y)(k'{}^\top y).
$
If $\rank(V)>1$, two nonproportional rows of $V$ determine the common
factor $k^\top y$, so the projective fiber is again a singleton.

Finally, if $\rank(V)=1$, write $V=hv^\top$. Up to scaling, the only
possibilities are
$
(A',V')=(kq^\top,hv^\top)
$
or
$
(A',V')=(vq^\top,hk^\top).
$
Thus every projective fiber contains at most two points. Since the
source is projective, the map is proper; finite fibers therefore imply
that it is finite.

If $a>1$, a general $A\in D_a^\CC$ has rank greater than one, so the
general fiber is a singleton. If $a=1$ and $d'>1$, a general $V$ has
rank greater than one, so again the general fiber is a singleton.
Hence the map is birational in these cases. If $a=d'=1$, then for
general nonproportional $k$ and $v$ the two projective preimages above
are distinct, so the map is generically two-to-one.
\end{proof}

We also recall the following standard criterion. It is stated in this
form in~\cite[Lemma~3.2]{kohn2017secants}; see also
\cite[Theorem~14.9]{harris1992algebraic} for the classical criterion
for finite morphisms underlying the result.

\begin{lemma}\label{lem:kohn-finite-birational}
Let $f:X\to Y$ be a finite birational surjective morphism between
irreducible projective varieties. Then $Y$ is smooth at $y$ if and
only if $f^{-1}(y)$ consists of a single point $x$, the variety $X$
is smooth at $x$, and $d_xf$ is injective.
\end{lemma}

\begin{proof}[Proof of Lemma~\ref{lem:multi-token-singular}]
Suppose first that $1<a<d$. If
$p\in\PP(\mathcal X_{a-1}^\CC)$, then $p$ has a preimage with
$\rank(A)<a$. Since
$
\operatorname{Sing}(D_a^\CC)=D_{a-1}^\CC,
$
the source is singular at this preimage. Hence $p$ is singular by
Lemma~\ref{lem:kohn-finite-birational}.

Conversely, let $p\notin\PP(\mathcal X_{a-1}^\CC)$. Its unique
projective preimage has $\rank(A)=a>1$ by
Proposition~\ref{prop:multi-token-finite-birational}, and the source is
smooth there. The proof of
Lemma~\ref{lem:attention-differential-t>1} is algebraic and applies
over $\CC$; since $\rank(A)>1$, its affine kernel consists only of the
scaling direction. Thus the projective differential is injective, and
$p$ is smooth by Lemma~\ref{lem:kohn-finite-birational}. Therefore
$
\operatorname{Sing}\bigl(\PP(\mathcal X_a^\CC)\bigr)
=
\PP(\mathcal X_{a-1}^\CC).
$

Now suppose that $a=d$. The parameter space $D_d^\CC$ is smooth. Let
$p\in\PP(Z^\CC)$, with $A=kq^\top$ and $V=hv^\top$. If $k$ and $v$
are not proportional,
Proposition~\ref{prop:multi-token-finite-birational} gives two distinct
projective preimages. If $k$ and $v$ are proportional, the fiber is a
singleton, but Lemma~\ref{lem:attention-differential-t>1} gives a
$d$-dimensional affine kernel, so the projective differential is not
injective. Thus $p$ is singular in either case.

If $p\notin\PP(Z^\CC)$, its projective fiber is a singleton and the
exceptional case in Lemma~\ref{lem:attention-differential-t>1} does
not occur. The source is smooth and the projective differential is
injective, so $p$ is smooth. Therefore
$
\operatorname{Sing}\bigl(\PP(\mathcal X_d^\CC)\bigr)
=
\PP(Z^\CC).
$

Finally, suppose that $a=1$ and $d'>1$. The projective determinantal
variety $\PP(D_1^\CC)$ is smooth, and
Proposition~\ref{prop:multi-token-finite-birational} shows that the
projectivized parametrization is finite and birational.

Let $p\in\PP(Z^\CC)$ and write $A=kq^\top$, $V=hv^\top$. If $k$ and
$v$ are not proportional, the fiber contains the two distinct
projective points obtained by exchanging $k$ and $v$. If they are
proportional, the fiber is a singleton, but
Lemma~\ref{lem:attention-differential-t>1} gives a $d$-dimensional
affine kernel. Thus $p$ is singular.

Conversely, if $p\notin\PP(Z^\CC)$, then $\rank(V)>1$. The projective
fiber is a singleton and the affine kernel of the differential is only
the scaling direction. Hence $p$ is smooth, and
$
\operatorname{Sing}\bigl(\PP(\mathcal X_1^\CC)\bigr)
=
\PP(Z^\CC).
$
\end{proof}

The work of~\citep{HMK} characterizes the nontrivial fibers and critical
points of the attention parametrization, which account for two possible
sources of singularities. When $a<d$, there is an additional source:
the parameter space $D_a^\CC$ is itself singular, with
$
\operatorname{Sing}(D_a^\CC)=D_{a-1}^\CC.
$
Together with the fiber and differential analysis above, this gives the
singular-locus description in Lemma~\ref{lem:multi-token-singular}.

For $1<a<d$, the corresponding real singular locus is
$\mathcal X_{a-1}$, while for $a=d$ it is contained in $\mathcal X_1$.
Thus, in both cases, the absence of proximity bias toward the real
singular locus follows from
Theorem~\ref{thm:multi-token-lower-rank}. When $a=1$ and $d'>1$, the
nonzero real singular locus is the real counterpart $Z$ of $Z^\CC$;
the absence of proximity bias toward this locus is proved next.

\subsection{The rank-one model with \texorpdfstring{$d'>1$}{d'>1}}

\multitokenrankonenoBias*

\begin{proof}[Proof of Proposition~\ref{prop:multi-token-rank-one-no-bias}]
Since $a=1$, the projective determinantal variety $\PP(D_1^\CC)$ is
smooth, and hence so is the source of the projectivized parametrization.
Proposition~\ref{prop:multi-token-finite-birational} shows that
the projectivized parametrization is finite and birational.

Let $p\in\PP(Z^\CC)$ and write $A=kq^\top$, $V=hv^\top$. If $k$ and
$v$ are not proportional, the fiber contains the two distinct
projective points obtained by exchanging $k$ and $v$. If they are
proportional, the fiber is a singleton, but
Lemma~\ref{lem:attention-differential-t>1} gives a $d$-dimensional
affine kernel. Since $d>1$, the projective differential is therefore
not injective. Thus $p$ is singular by
Lemma~\ref{lem:kohn-finite-birational}.

Conversely, if $p\notin\PP(Z^\CC)$, then $\rank(V)>1$. The projective
fiber is a singleton and the affine kernel of the differential is only
the scaling direction, so the projective differential is injective.
Hence $p$ is smooth, and
$
\operatorname{Sing}(\PP(\cX_1^\CC))=\PP(Z^\CC).
$

We now return to the real model. Let $Z$ denote the real counterpart of
$Z^\CC$ and retain the locus $Y$ defined in \Eqref{eq:rank-one-attention}. 
For any real preimage
$(A,V)$ of $p$, the image of the differential restricted to
$T_{D_1}(A)\times\RR^{d'\times d}$ is contained in $T_{\cX_1}(p)$.
Consequently,
$$
N_{\cX_1}(p)
\subseteq
\operatorname{im}\left(
d\phi|_{T_{D_1}(A)\times\RR^{d'\times d}}
\right)^\perp.
$$

At a point $p\in Z\setminus Y$, choose a parametrization with
$A=kq^\top$ and $V=hv^\top$, where $v$ is not proportional to $k$.
The tangent space of $D_1$ at $A$ has dimension $2d-1$, and the
restricted differential has only the one-dimensional scaling kernel.
Its image therefore has dimension $2d+dd'-2=\dim \cX_1$. Since
$\dim Z=3d+d'-3$,
\begin{align*}
\dim\left(\bigcup_{p\in Z\setminus Y}(p+N_{\cX_1}(p))\right)
&\leq 3d+d'-3+n-(2d+dd'-2)=n-(d-1)(d'-1)<n.
\end{align*}

For $p\in Y$, the $d$-dimensional kernel from
Lemma~\ref{lem:attention-differential-t>1} lies in
$T_{D_1}(A)\times\RR^{d'\times d}$. Hence the restricted differential
has rank $d+dd'-1$. Since $\dim Y\leq d'+2d-2$,
\begin{align*}
\dim\left(\bigcup_{p\in Y}(p+N_{\cX_1}(p))\right)
&\leq d'+2d-2+n-(d+dd'-1)=n-(d-1)(d'-1)<n.
\end{align*}

Finally, since $\cX_1$ is a symmetric cone, we get
$
N_{\cX_1}(0)=\operatorname{span}(\cX_1)^\perp,
$
which is a proper linear subspace. Thus the union of the translated
normal cones over the real singular locus is not full-dimensional, so
there is no proximity bias toward the real singular locus.
\end{proof}

\subsection{The scalar-output rank-one model}

When $d'=1$, the dimension-counting argument of the previous subsection
no longer rules out proximity bias: the upper bounds there become
$n-(d-1)(d'-1)=n$. Moreover, in this case the parametrization is
generically two-to-one due to the symmetry exchanging $k$ and $v$, and
at the locus where $k$ and $v$ are proportional its differential has a
$d$-dimensional kernel. Thus these parameters do not give local
coordinates on the model at the points of interest. We therefore
introduce an injective parametrization that removes this symmetry. Write
$$
\psi(q,S)(X)
=
\left(\sum_{j=1}^t x_j^\top Sx_j\right)q^\top X,
$$
and define $U = \{\psi(q,S):q\in\RR^d,\ \rank(S)\leq1\}$
and $U^\CC=\{\psi(q,S):q\in\CC^d,\ \rank(S)\leq1\}$.

\multitokenrankone*

Figure~\ref{fig:attention-rank-one-geometry} illustrates the two
geometric regimes in the smallest cases $d=2$ and $d=3$.
As shown in Lemma~\ref{lem:multi-token-symmetric-quotient}, the map
$q\otimes S\mapsto\psi(q,S)$ is injective for $t>1$, so these
low-dimensional slices faithfully represent the corresponding geometry
of $\mathcal X_1$. For $d=2$, the locus $U$ lies on the boundary of the
real model despite the algebraic model being smooth, whereas for $d=3$
it appears as a singular locus.

\begin{lemma}\label{lem:multi-token-symmetric-quotient}
Suppose that $a=d'=1$. Write $A=kq^\top$, $V=v^\top$, and set
$S=(vk^\top+kv^\top)/2$. Then $\phi(A,V)=\psi(q,S)$.

For $\mathbb F=\RR$ or $\CC$, the map
$q\otimes S\mapsto\psi(q,S)$ is an injective linear map from
$\mathbb F^d\otimes S^2(\mathbb F^d)$ into the ambient function space.
Over $\CC$, the matrices $S$ arising in this way are exactly the
symmetric matrices of rank at most two. Over $\RR$, they are exactly
the symmetric matrices having at most one positive and at most one
negative eigenvalue.
\end{lemma}

\begin{proof}
The formula for $\phi$ follows from
$v^\top XX^\top k=\sum_j(v^\top x_j)(k^\top x_j)$.

For injectivity, consider two distinct tokens. In the output coordinate
corresponding to the first token, the terms linear in $x_1$ and
quadratic in $x_2$ are
$(q^\top x_1)(x_2^\top Sx_2)$, whose coefficients recover $q\otimes S$.

Clearly $\rank(S)\leq2$. Conversely, over $\CC$, every symmetric matrix
of rank at most two can be written as $S=uu^\top+ww^\top$. Taking
$v=u+iw$ and $k=u-iw$ gives
$S=(vk^\top+kv^\top)/2$.

Over $\RR$, set $p=(v+k)/2$ and $r=(v-k)/2$. Then
$S=pp^\top-rr^\top$, so $S$ has at most one positive and at most one
negative eigenvalue. Conversely, every symmetric matrix with this
property can be written as $pp^\top-rr^\top$, and hence arises from
$v=p+r$ and $k=p-r$.
\end{proof}

Set
$
\mathcal D_2^{\mathrm{sym}}
=
\{S\in S^2(\CC^d):\rank(S)\leq2\}.
$
By Lemma~\ref{lem:multi-token-symmetric-quotient},
$\PP(\cX_1^\CC)$ is linearly isomorphic to the Segre embedding of
$\PP^{d-1}\times\PP(\mathcal D_2^{\mathrm{sym}})$.

\begin{lemma}\label{lem:multi-token-symmetric-rank-two-singular}
If $d>2$, then the singular locus of
$\PP(\mathcal D_2^{\mathrm{sym}})$ is its rank-one locus. If $d=2$,
$\PP(\mathcal D_2^{\mathrm{sym}})$ is smooth.
\end{lemma}

\begin{proof}
If $d=2$, every symmetric matrix has rank at most two, so
$\mathcal D_2^{\mathrm{sym}}=S^2(\CC^2)$. Suppose that $d>2$. At a rank-two matrix $S$, choose coordinates with
$S=\left(\begin{smallmatrix}S_0&0\\0&0\end{smallmatrix}\right)$,
where $S_0$ is invertible. A nearby symmetric matrix can be written as
$\left(\begin{smallmatrix}P&Q\\Q^\top&R\end{smallmatrix}\right)$.
Since $P$ remains invertible, the rank-at-most-two condition is locally
equivalent to $R=Q^\top P^{-1}Q$. Hence the variety is smooth at~$S$. The variety is defined by its $3\times3$ minors. At a matrix of rank at
most one, the differential of every such minor vanishes, since every
partial derivative is a $2\times2$ minor. Thus the Zariski tangent
space is the full ambient space. Since
$\mathcal D_2^{\mathrm{sym}}$ is proper for $d>2$, every projective
rank-one point is singular.
\end{proof}

We now return to the real model. Let
$$
\mathcal R_1
=
\{S\in S^2(\RR^d):
S\text{ has at most one positive and at most one negative eigenvalue}\}.
$$
By Lemma~\ref{lem:multi-token-symmetric-quotient},
$\cX_1=\{\psi(q,S):q\in\RR^d,\ S\in\mathcal R_1\}$.

\begin{lemma}\label{lem:multi-token-rank-one-normal}
Let $S\in\mathcal R_1$ be positive semidefinite of rank one and
$K=\ker(S)$, so $\dim K=d-1$. Since $K^\perp$ is one-dimensional,
with respect to $\RR^d=K^\perp\oplus K$ write
\[
S=\begin{pmatrix}
s_0&0\\
0&0
\end{pmatrix},
\qquad s_0>0.
\]
If $B_{22}$ denotes the $K\times K$ block of $B$, then
$$
T_{\mathcal R_1}(S)
=
\{B:B_{22}\preceq0,\ \rank(B_{22})\leq1\},
$$
and
$$
N_{\mathcal R_1}(S)
=
\{W:W=P_KWP_K,\ W|_K\succeq0\}.
$$
\end{lemma}

\begin{proof}
Write a nearby matrix as
$S'=\left(\begin{smallmatrix}P&Q\\Q^\top&R\end{smallmatrix}\right)$.
Since $s_0>0$, the block $P$ remains positive. By the Schur complement
formula, the remaining positive and negative eigenvalues of $S'$ are
those of
$C=R-Q^\top P^{-1}Q$. Since $S'\in\mathcal R_1$, the matrix $C$ must
be negative semidefinite and have rank at most one. Taking first-order
terms gives $B_{22}\preceq0$ and $\rank(B_{22})\leq1$.

Conversely, if $B_{22}$ satisfies these conditions, then
$$
S(t)=
\begin{pmatrix}
s_0+tB_{11} & tB_{12}\\
tB_{12}^\top &
tB_{22}+t^2B_{12}^\top(s_0+tB_{11})^{-1}B_{12}
\end{pmatrix}
$$
lies in $\mathcal R_1$ for small $t>0$ and has derivative $B$ at zero.

For the normal cone, the $B_{11}$ and $B_{12}$ blocks vary with both
signs, so a normal vector must be supported on $K$. The remaining
tangent directions contain $B_{22}=-zz^\top$ for every $z\in K$.
Polarity therefore gives $W|_K\succeq0$.
\end{proof}

\begin{proposition}\label{prop:multi-token-rank-one-bias}
The union of the Voronoi cells of points in $U$ contains a nonempty open
subset of the ambient space; that is,
$
\operatorname{int}\left(
\bigcup_{p\in U}\operatorname{Vor}_{\cX_1}(p)
\right)\neq\varnothing.
$
Hence there is proximity bias toward $U$.
\end{proposition}

\begin{proof}
The proof proceeds in three steps. First, fix $p=\psi(q,S)\in U\setminus{0}$ with $S\succeq0$ of rank one and let $K=\ker(S)$. For nearby points of $\cX_1$, the $K\times K$ block of the Schur complement is negative semidefinite; pairing it with a positive definite matrix $W$ supported on $K$ produces directions $w$ such that $p+w\in\operatorname{Vor}_{\cX_1}(p)$. Second, we show that, as $W$ varies, these directions fill an open subset of the normal space $T_U(p)^\perp$. Finally, allowing $p$ to vary in $U$ and applying the inverse function theorem to $(p,w)\mapsto p+w$, we show that the corresponding Voronoi cells sweep out a nonempty open subset of the ambient space.

Consider the unrestricted symmetric model
$
\widetilde X
=
\{\psi(q,S):q\in\RR^d,\ S\in S^2(\RR^d)\}.
$
At a nonzero pair $(q,S)$, the kernel of $d\psi_{(q,S)}$ is the one-dimensional scaling direction $(q,-S)$.
Thus, after fixing the scaling locally, $(q,S)\mapsto\psi(q,S)$ is a
local diffeomorphism onto~$\widetilde X$.

Fix $p\in U\setminus\{0\}$ and write $p=\psi(q,S)$ with
$q\neq0$ and $S\succeq0$ of rank one. Let $K=\ker(S)$ and choose
$W\in S^2(\RR^d)$ supported on $K$ with $W|_K\succ0$. There is a unique
vector $w_W\in T_{\widetilde X}(p)$ satisfying
$
\langle w_W,d\psi_{(q,S)}(\dot q,B)\rangle
=
\langle W,B\rangle.
$
Indeed, the right-hand side vanishes on the scaling direction since
$\langle W,S\rangle=0$.

We show that, for sufficiently small perturbations $w$ of $w_W$ in the
directions normal to $\widetilde X$, the point $p+w$ belongs to
$\operatorname{Vor}_{\cX_1}(p)$. Let
$x=\psi(q',S')\in \cX_1$ be sufficiently close to $p$, with $(q',S')$
chosen close to $(q,S)$. With respect to $K^\perp\oplus K$, write
$$
S'=
\begin{pmatrix}
P&Q\\
Q^\top&R
\end{pmatrix}.
$$
Since $S'\in\mathcal R_1$ and $P>0$, its Schur complement
$
C=R-Q^\top P^{-1}Q
$
satisfies $C\preceq0$. Hence
\begin{align*}
\langle W,S'-S\rangle
&=
\langle W,C\rangle
+
\langle W,Q^\top P^{-1}Q\rangle
\leq
c_0\|S'-S\|^2
\end{align*}
for some $c_0>0$. Taylor expansion of $\psi$ and the local inverse then
give
$
\langle w_W,x-p\rangle
\leq
c\|x-p\|^2
$
for all $x\in \cX_1$ sufficiently close to $p$. The same estimate holds for
$
w=w_W+\eta,
$
with
$
\eta\in T_{\widetilde X}(p)^\perp,
$
since $\widetilde X$ is smooth at $p$ and hence
$
\langle\eta,x-p\rangle=O(\|x-p\|^2).
$
Therefore, for sufficiently small such $w$ and every
$x\in \cX_1\setminus\{p\}$ sufficiently close to $p$, we get
\begin{align*}
\|x-(p+w)\|^2-\|p-(p+w)\|^2
&=
\|x-p\|^2-2\langle w,x-p\rangle>0.
\end{align*}
For sufficiently small $w$, points of $\cX_1$ outside a fixed neighborhood
of $p$ are also farther from $p+w$ than $p$. Thus
$
p+w\in\operatorname{Vor}_{\cX_1}(p).
$

It remains to show that these subsets of the Voronoi cells sweep out an
open subset of the ambient space as $p$ varies in $U\setminus\{0\}$.
Since
$
\dim S^2(\RR^d)=\binom{d+1}{2},
$
we have
$
\dim\widetilde X
=
d+\binom{d+1}{2}-1
$
and
$
\dim U=2d-1.
$
Therefore
$$
\operatorname{codim}_{\widetilde X}(U)
=
\binom{d+1}{2}-d
=
\binom d2
=
\dim S^2(K).
$$
The map $W\mapsto w_W$ is injective and tangent directions to
$U$ have zero $K\times K$ block, so
$
w_W\in T_U(p)^\perp\cap T_{\widetilde X}(p).
$
The dimension equality above therefore implies that, as $W$ varies,
the vectors $w_W$ fill
$
T_U(p)^\perp\cap T_{\widetilde X}(p).
$
Allowing $\eta\in T_{\widetilde X}(p)^\perp$ gives all of
$T_U(p)^\perp$.

Since the condition $W|_K\succ0$ is open, the preceding construction
gives an open subset of the normal bundle of $U\setminus\{0\}$ consisting
of sufficiently small vectors $w$ such that
$
p+w\in\operatorname{Vor}_{\cX_1}(p).
$
Consider the addition map
$
\Theta(p,w)=p+w.
$
Along the zero section,
$
d\Theta_{(p,0)}(v,\xi)=v+\xi,
$
which is an isomorphism since
$
T_U(p)\oplus T_U(p)^\perp
$
is the ambient space. Hence $d\Theta$ remains invertible at sufficiently
small admissible nonzero vectors. By the inverse function theorem, the
image of these vectors contains a nonempty open subset of the ambient
space. Since every point in this image belongs to
$\operatorname{Vor}_{\cX_1}(p)$ for some $p\in U$, we obtain
$
\operatorname{int}\left(
\bigcup_{p\in U}\operatorname{Vor}_{\cX_1}(p)
\right)\neq\varnothing.
$
\end{proof}

\begin{proof}[Proof of Theorem~\ref{thm:multi-token-rank-one}]
By Lemma~\ref{lem:multi-token-symmetric-quotient},
$\PP(\cX_1^\CC)$ is linearly isomorphic to the Segre embedding of
$\PP^{d-1}\times\PP(\mathcal D_2^{\mathrm{sym}})$. If $d>2$, Lemma~\ref{lem:multi-token-symmetric-rank-two-singular}
shows that the singular locus is exactly the image of the rank-one
locus in $\mathcal D_2^{\mathrm{sym}}$, namely $\PP(U^\CC)$.
If $d=2$, then
$\mathcal D_2^{\mathrm{sym}}=S^2(\CC^2)$, so
$\PP(\cX_1^\CC)$ is smooth.

Over $\RR$, the model is parametrized by $S\in\mathcal R_1$. At every
positive semidefinite rank-one matrix,
Lemma~\ref{lem:multi-token-rank-one-normal} gives a nontrivial
one-sided normal cone, so these points lie on the boundary of the real
model. Replacing $S$ by $-S$ gives the same conclusion for the negative
semidefinite rank-one matrices. Thus the nonzero locus $U$ is a boundary
of the real model. Finally,
Proposition~\ref{prop:multi-token-rank-one-bias} gives proximity bias
toward $U$.
\end{proof}

\section{Proofs of the Single-Token Results}\label{app:single-token-attention}
Throughout this section, assume that $t=1$. Writing
$S=\frac12(A+A^\top)$, the attention map takes the form $\psi(S,V)(x)=(Vx)(x^\top Sx).$

\singletokenattention*

\subsection{Factorization singularities}

The main difference from the multiple-token case is that the relevant
factor is now the quadratic form $Q_S(x)=x^\top Sx$.

\begin{lemma}\label{lem:single-token-factorization}
Let $S\in S^2(\CC^d)$.
The quadratic form $Q_S$ is reducible if and only if
$\rank(S)\leq2$.

Every nonzero projective fiber of $\psi$ is finite. Moreover, the
projective fiber is a singleton unless $\rank(V)=1$ and $\rank(S)\leq2$. If $V=hv^\top$, the differential drops rank precisely when $v^\top x$ divides $Q_S(x)$. In that case the affine kernel has
dimension $d$; otherwise it consists only of the scaling direction.
\end{lemma}

\begin{proof}
A quadratic form over $\CC$ factors into two linear forms precisely
when its symmetric matrix has rank at most two. Now suppose $\psi(S,V)=\psi(S',V')\neq0.$
If $\rank(V)>1$, two nonproportional rows of $V$, together with unique
factorization, determine $Q_S$ and hence $(S,V)$ up to scaling.
If $\rank(V)=1$, write $V=hv^\top$. When $Q_S$ is irreducible,
unique factorization again gives only the scaling ambiguity.
Thus nontrivial fibers can occur only when
$\rank(V)=1$ and $\rank(S)\leq2$. When $\rank(V)=1$ and $Q_S$ is reducible, unique factorization shows that any other preimage is obtained by assigning one of the finitely
many linear factors of $(v^\top x)Q_S(x)$ to $V'$, so the projective
fiber is finite.

For the differential, write $\dot v_i^\top$ for the rows of $\dot V$.
The kernel equation is
\[
(\dot v_i^\top x)Q_S(x)+(v_i^\top x)Q_B(x)=0.
\]
The same factorization argument gives only the scaling direction
unless $V=hv^\top$ and $v^\top x$ divides $Q_S(x)$. In the latter
case, writing
$Q_S(x)=(v^\top x)(w^\top x)$, the kernel is parametrized by
$r\in\CC^d$ through
$
\dot V=hr^\top
$ and $
Q_B(x)=-(r^\top x)(w^\top x),
$
and therefore has dimension $d$.
\end{proof}

Over $\CC$, let $Z^\CC=\{\psi(S,V):\rank(V)=1,\ \rank(S)\leq2\}$, and let $Z$ denote the corresponding locus over $\RR$.

\begin{proposition}\label{prop:single-token-full-rank}
Assume that $a=d>1$ and $d'>1$. Then
$
\operatorname{Sing}\bigl(\PP(\mathcal X_d^\CC)\bigr)=\PP(Z^\CC)
$
and there is no proximity bias toward the corresponding real singular locus $Z$.
\end{proposition}

\begin{proof}
By Lemma~\ref{lem:single-token-factorization}, the projectivized
parametrization has finite fibers and is generically one-to-one.
Since its source is smooth, the finite-birational criterion from
Lemma~\ref{lem:kohn-finite-birational} shows that the singular points are exactly
those in $\PP(Z^\CC)$: outside $Z^\CC$ the fiber is a singleton and
the projective differential is injective, whereas on $Z^\CC$ there
is either a nontrivial fiber or a critical preimage.

We now work over $\RR$. For $p\in\mathcal X_d$, set
$\mathcal L_p
=
p+\bigcap_{(S,V)\in\psi^{-1}(p)}
\operatorname{im}(d\psi_{(S,V)})^\perp$.
Since every image of the differential is contained in
$T_{\mathcal X_d}(p)$, we have
$\Vorloc_{\mathcal X_d}(p)
\subseteq p+N_{\mathcal X_d}(p)\subseteq\mathcal L_p$.

To bound $\mathcal L_p$ over the real singular locus $Z$, we separate
the exceptional locus
$Y=\{\psi(v^{\otimes2},hv^\top):v\neq0,\ h\neq0\}\subseteq Z$.
For points in $Z\setminus Y$, there is always a noncritical preimage,
which gives the stronger dimension bound below. On $Y$, the unique
projective preimage is critical, so a separate bound is needed.

Set $n=d'\binom{d+2}{3}$ and
$m=dd'+\binom{d+1}{2}$, so that $\dim\mathcal X_d=m-1$.

Let $p\in Z\setminus(Y\cup\{0\})$. If $p$ already has a noncritical
preimage, there is nothing to show. Otherwise, choose a critical
preimage $(S,V)$. By Lemma~\ref{lem:single-token-factorization}, we
may write $V=hv^\top$ and
$Q_S(x)=(v^\top x)(w^\top x)$. Since $p\notin Y$, we have
$w\not\parallel v$. The same output then has a parametrization with
$V'=hw^\top$ and $Q_{S'}(x)=(v^\top x)^2$, which is noncritical.
Hence $\dim\mathcal L_p\leq n-(m-1)$. Since
$\dim Z\leq d'+3d-3$,
$$
\dim\left(
\bigcup_{p\in Z\setminus(Y\cup\{0\})}\mathcal L_p
\right)
\leq
n-\frac{(d-1)(2d'+d-4)}{2}
<n.
$$

For $p\in Y$, the projective preimage is unique and
Lemma~\ref{lem:single-token-factorization} gives a $d$-dimensional
affine kernel. Thus $\dim\mathcal L_p=n-(m-d)$. Since
$\dim Y=d'+d-1$,
$$
\dim\left(
\bigcup_{p\in Y}\mathcal L_p
\right)
\leq
n-\frac{(d-1)(2d'+d-2)}{2}
<n.
$$
Finally, since $\mathcal X_d$ is a symmetric cone, its local Voronoi cell
at the origin is contained in a proper linear subspace. Hence the union of
the local Voronoi cells over $Z$ has empty interior.
\end{proof}

\subsection{Rank-constrained models}

The second new feature in the regime $t=1$ is that symmetrization interacts differently
with the rank constraint over $\CC$ and over $\RR$.

\begin{proposition}\label{prop:single-token-symmetric-parts}
Let $1\leq a<d$ and let $S=\frac12(A+A^\top)$. Over $\CC$, the
possible symmetric parts of matrices satisfying $\rank(A)\leq a$ are
exactly
$\{S\in S^2(\CC^d):\rank(S)\leq\min\{2a,d\}\}$.
Over $\RR$, they are exactly
$\mathcal R_a=\{S\in S^2(\RR^d):n_+(S)\leq a,\ n_-(S)\leq a\}$, where $n_+(S)$ and $n_-(S)$ denote the numbers of positive and negative eigenvalues of $S$, respectively.
\end{proposition}

\begin{proof}
The inclusion $\rank(S)\leq2a$ over $\CC$ is immediate. Conversely,
write $S=\sum_{i=1}^r u_i u_i^\top$ with $r\leq2a$. Pairing the
summands and using
$\operatorname{Sym}((u+iv)(u-iv)^\top)=uu^\top+vv^\top$
expresses $S$ as the symmetric part of a sum of at most
$\lceil r/2\rceil\leq a$ rank-one matrices.

Over $\RR$, if $A=K^\top Q$ with $K,Q\in\RR^{a\times d}$, then
$\operatorname{Sym}(A)=PP^\top-RR^\top$ for suitable
$P,R\in\RR^{d\times a}$, so it has at most $a$ positive and at most
$a$ negative eigenvalues. The converse follows by reversing this
construction.
\end{proof}

If $2a\geq d$, the complexified model is therefore the same as in the
unconstrained case. The genuinely new regime is $2a<d$.

\begin{proposition}\label{prop:single-token-constrained-singular}
Assume that $a>1$ and $2a<d$. Then
$
\operatorname{Sing}\bigl(\PP(\mathcal X_a^\CC)\bigr)
=
\PP\bigl(
\{\psi(S,V):\rank(S)\leq2a-1\}
\bigr).
$
\end{proposition}

\begin{proof}
Proposition~\ref{prop:single-token-symmetric-parts} implies that the complex parameter space for $S$ is the symmetric determinantal variety
$\mathcal D_{2a}=\{S\in S^2(\CC^d):\rank(S)\leq2a\}$, whose singular
locus is $\mathcal D_{2a-1}$.

Outside $\mathcal D_{2a-1}$, the matrix $S$ has rank $2a>2$, so
$Q_S$ is irreducible. Hence the projective fiber is a singleton and
the projective differential is injective. By Lemma~\ref{lem:single-token-factorization}, the projectivized parametrization has finite fibers and is generically one-to-one,
hence is finite and birational. The finite-birational
criterion \ref{lem:kohn-finite-birational} therefore gives smoothness there. Conversely, every point
in the displayed locus has a preimage in the singular locus
$\mathcal D_{2a-1}$ of the parameter space, and is therefore
singular.
\end{proof}

Let $U$ be the locus of points $\psi(S,V)$ with $V\neq0$ and
$\operatorname{inertia}(S)=(a,a-1,d-2a+1)$. Since $\psi(S, V) = \psi(-S, -V)$, inertia $(a-1, a, d-2a+1)$ gives the same set.

\begin{lemma}\label{lem:single-token-boundary-normal}
Let $S\in\mathcal R_a$ have inertia $(a,a-1,k)$, where
$k=d-2a+1$, and let $K=\ker(S)$. With respect to
$\RR^d=K^\perp\oplus K$, write
$S=\left(\begin{smallmatrix}S_0&0\\0&0\end{smallmatrix}\right)$,
and let $B_{22}$ denote the $K\times K$ block of $B$. Then
$$
T_{\mathcal R_a}(S)
=
\{B:B_{22}\preceq0,\ \rank(B_{22})\leq1\},
\qquad
N_{\mathcal R_a}(S)
=
\{W:W=P_KWP_K,\ W|_K\succeq0\}.
$$
For inertia $(a-1,a,k)$, the inequalities are reversed.
\end{lemma}

\begin{proof}
For a nearby matrix
$S'=\left(\begin{smallmatrix}P&Q\\Q^\top&R\end{smallmatrix}\right)$,
the block $P$ has inertia $(a,a-1)$. Hence the inertia constraint on
$S'$ is, to first order, equivalent to requiring the Schur complement
$R-Q^\top P^{-1}Q$ to be negative semidefinite of rank at most one.
Conversely, if $B_{22}\preceq0$ and $\rank(B_{22})\leq1$, choosing
a curve whose Schur complement is $tB_{22}$ realizes $B$ as a
tangent direction. Taking the polar gives
the normal cone. The other case follows by replacing $S$ by $-S$.
\end{proof}

\begin{proposition}\label{prop:single-token-boundary-bias}
Assume that $a>1$ and $2a<d$. There is proximity bias toward
$U$.
\end{proposition}

\begin{proof}
The argument is the same normal-bundle argument used in
Proposition~\ref{prop:multi-token-rank-one-bias}.

Fix $p=\psi(S,V)\in U$ and let $K=\ker(S)$. Since
$\rank(S)=2a-1>2$, the unrestricted symmetric parametrization is,
after quotienting by scaling, a local diffeomorphism near $p$.
Choose $W$ supported on $K$ with $W|_K\succ0$, and transport it
through this local parametrization to a normal direction $w_W$ at
$p$. By Lemma~\ref{lem:single-token-boundary-normal} and the Schur
complement, for $x\in\mathcal X_a$ sufficiently close to $p$ we have
$\langle w_W,x-p\rangle\leq c\|x-p\|^2$. Hence sufficiently small
points in the direction $w_W$ belong to
$\operatorname{Vor}_{\mathcal X_a}(p)$.

Writing $k=d-2a+1$, the admissible matrices $W$ form an open subset
of $S^2(K)$, and
$\dim S^2(K)=\binom{k+1}{2}
=\operatorname{codim}_{\mathcal X_d}(U)$.
Together with the directions normal to $\mathcal X_d$, these
directions fill $T_{U}(p)^\perp$. The inverse-function-theorem
argument from Proposition~\ref{prop:multi-token-rank-one-bias}
therefore shows that
$
\bigcup_{p\in U}\operatorname{Vor}_{\mathcal X_a}(p)
$
contains a nonempty open subset of the ambient space.
\end{proof}

Propositions~\ref{prop:single-token-full-rank},
\ref{prop:single-token-constrained-singular}, and
\ref{prop:single-token-boundary-bias}
prove Theorem~\ref{thm:single-token-attention}.

\section{Numerical evidence for Conjecture~\ref{conj:mlp-small-dim-no-bias}} \label{sec:numerical-evidence-conjecture}

Let $n=\dim\cH$. For $f\in\cF'$, let $m(f)$ be the dimension of the span of $\operatorname{im}d\Phi_\theta$ over all
parameters $\theta$ of the full network with $f_\theta=f$, where $\Phi: \theta \mapsto f_\theta$ is the parametrization of $\cF$. The normal cone $N_\cF(f)$ lies in a space of dimension
$n-m(f)$, by \Cref{rem:bound-normal-parametrization}.
So there is no proximity bias toward $\cF'$ if $m(f)$ exceeds the dimension of every subset of
$\cF'$ on which it is constant, by \Cref{lem:count}.
At a generic point of $\cF'$, this means $m(f)>\dim\cF'$.

We computed $m(f)$ at random parameters, with exact linear algebra modulo a large prime. To explore the fiber of $f$,
we gave the inactive neurons random incoming weights and merged constant or proportional neurons into the next layer.

We tested 10,768 networks with bias terms: $d_0\le2$, $L\in\{3,4\}$, hidden widths at most $3$, $d_L\le2$, activations
$t^2$,~$t^3$ or generic quadratics and cubics, and one or two inactive neurons.

\begin{center}\small
\begin{tabular}{lrrr}
\toprule
inactive neurons & cases & $\dim\cF'<\dim\cF$ & with $m>\dim\cF'$\\
\midrule
one, in one layer & 5088 & 2884 & 2884\\
two, in one layer & 2544 & 2148 & 2148\\
one in each of two layers & 3136 & 2666 & 2666\\
\bottomrule
\end{tabular}
\end{center}

The condition holds in all 7,698 cases with $\dim\cF'<\dim\cF$. In the 3,070 cases with $\dim\cF'=\dim\cF$, where there
is proximity bias by \Cref{prop:equal-dimension-bias},
we found $m=\dim\cF'$.

We also tested 40,000 degenerate subsets of $\cF'$, taken from 1,600 of these cases. They are defined by zero weights,
inactive or dead neurons, equal, opposite or proportional neurons, or zero biases. The condition holds on all of them,
except those consisting of constant functions, where $N_\cF=\{0\}$ because of the output bias.

The ranks are exact with high probability but not certified, and only some degenerate subsets were tested. They provide evidence  for the conjecture. For a fixed small architecture, an exact symbolic computation of the strata of $\cF'$ on which $m(f)$ is constant would give a proof.

\section{Details for figures}\label{sec:figures}

All three-dimensional figures show the intersection of a neural model with a three-dimensional affine slice of its
ambient space, clipped to a ball. The real model is drawn in blue and the submodel or singular locus in magenta;
dots are points of the Zariski closure that are not in the real model. We write
$\sym(B)=\frac12(B+B^\top)$.

\subsection{\Cref{fig:mlp}}

Shallow MLP $\mathbf d=(3,2,1)$, $\rho(t)=t^2$, without biases. Then $f(x)=x^\top Ax$ with $A=\sum_{i\le2}\lambda_iw_iw_i^\top$,
so $\cH=S^2(\RR^3)$, $\cF=\{\rank A\le2\}$, and the subnetwork $(3,1,1)$ is $\cF'=\{\rank A\le1\}$. We take the linear
slice
\[
A=S(x,y,z)=\begin{pmatrix}z&x&0\\x&x&y\\0&y&x\end{pmatrix}.
\]
Then $\cF\cap\{S(x,y,z)\}=\{z(x^2-y^2)=x^3\}$, the cone over a nodal cubic, and $\cF'\cap\{S(x,y,z)\}$ is the $z$-axis
$\{x=y=0\}$, along which the two sheets of the cone cross. The figure shows the ball of radius $2.4$ about the origin.

\subsection{\Cref{fig:implicit-bias-examples}}\label{sec:fig-examples}
See \Cref{fig:implicit-bias-examples-details} for a detailed description of \Cref{fig:implicit-bias-examples}.
\begin{figure}[h]
    \centering
    \includegraphics[width=\linewidth]{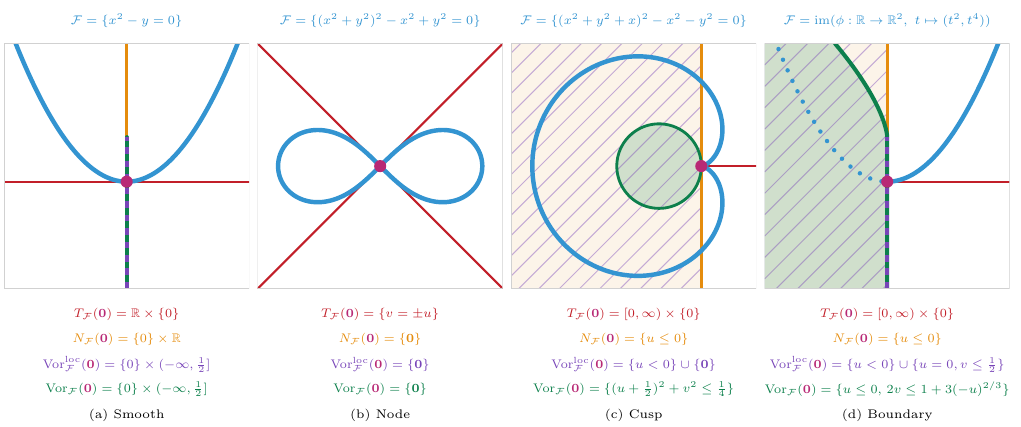}
    \caption{Details for \Cref{fig:implicit-bias-examples}.}
    \label{fig:implicit-bias-examples-details}
\end{figure}

\subsection{\Cref{fig:rank-one-cone}}

The single-neuron cone of architecture $\mathbf d=(2,1,1)$ and activation $\rho(t)=t^2$, without biases: $f(x)=v\,(w^\top x)^2$. Writing
$f=z\,x_1^2+2y\,x_1x_2+x\,x_2^2$, we have $\cH\cong\RR^3$ and
$
\cC=\{(x,y,z):\ xz=y^2\},
$
the cone of symmetric $2\times2$ matrices of rank at most one. Since
$\cC=-\cC$ and $\linspan\cC=\cH$, we have $N_\cC(0)=\{0\}$ and $\Vorloc_\cC(0)=\{0\}$.

\subsection{\Cref{fig:attention-rank-one-geometry}} \label{sec:fig-attention}

All panels use $d'=1$ and $t=2$; write $V=v^\top$. For $t\ge2$ the function $\phi(A,V)$ is determined, by an
injective linear map, by the tuple of symmetric matrices $\sym(Ae_qv^\top)$, $q\le d$; equivalently, by the tuple of
quadratic forms $Q_q(x)=(e_q^\top A^\top x)(v^\top x)$.

\paragraph{(a) $a=d=2$.} Here $\cH\cong S^2(\RR^2)^2\cong\RR^6$, and $\cX_2$ is the set of pairs $(Q_1,Q_2)$ of binary
quadratic forms with a common real linear factor. Its Zariski closure is the hypersurface $\Res(Q_1,Q_2)=0$, that is,
$
4\,\delta_{12}\delta_{23}-\delta_{13}^2=0,
$
where $\delta_{ij}$ is the minor in columns $i,j$ of the $2\times3$ matrix with rows
$\bigl((M_q)_{11},(M_q)_{12},(M_q)_{22}\bigr)$, $M_q=\sym(Ae_qv^\top)$. Its real points are $\cX_2$ together with the pairs $(\lambda S,\mu S)$ with
$S$ definite. The submodel $\cX_1=\{\rank A\le1\}$ consists of the pairs $(\lambda S,\mu S)$ with $\det S\le0$, and
it is the double curve of the surface. We take the affine slice
\[
M_1=\begin{pmatrix}-2y-z&y-1\\y-1&x-1\end{pmatrix},\qquad M_2=\begin{pmatrix}1+x&x-z\\x-z&x+y\end{pmatrix},
\]
in the ball of radius $2.4$ about $(-0.2,0.8,-0.5)$. The slice meets the double curve of proportional pairs in a
twisted cubic. It is drawn solid where $\det S<0$ (points of $\cX_1$) and dotted where $\det S>0$ (Zariski closure only).

\paragraph{(b) $a=1$, $d=2$.} By \Cref{lem:multi-token-symmetric-quotient},
$\cX_1\cong\{q\otimes S:\ q\in\RR^2,\ \det S\le0\}$. We write $q\otimes S$ as the
$2\times3$ matrix with rows $q_i\,(S_{11},S_{12},S_{22})$, so the Zariski closure of $\cX_1$ is the set of matrices of rank
at most one. We take the affine slice
\[
M(u,w,z)=\begin{pmatrix}u&w&u\\z&1-u&z\end{pmatrix}.
\]
In this slice, the Zariski closure is the hyperboloid $wz=u(1-u)$, the real model is its part
$u^2+z^2\le w^2+(1-u)^2$, and $U=\{\rank S\le1\}$ is the pair of lines $(r,r,1-r)$ and $(r,-r,r-1)$, $r\in\RR$. The
figure shows the ball of radius $2.6$ about $(\tfrac12,0,0)$.

\paragraph{(c) $a=1$, $d=3$.} We fix $q$ and take the linear slice $S=S(x,y,z)$ of $S^2(\RR^3)$ from G.1. The Zariski
closure is the cone $z(x^2-y^2)=x^3$. The real model, i.e.\ the matrices on this cone with at most one positive and one
negative eigenvalue, is its part $y^2\ge x^2$. The locus $U=\{\rank S\le1\}$ is the $z$-axis, which is a singular line of
the Zariski closure and a boundary of the real model. The figure shows the ball of radius $2.4$ about the origin.

\subsection{\Cref{fig:softmax-experiment}}

The setting is that of Section~4.3: $d=t=2$, $a=d'=1$, target $h=\psi(q_*,S_*)$ with $q_*$ a unit vector and
$S_*\succ0$ with eigenvalues $1$ and $0.2$, $512$ standard Gaussian inputs, initialization $k_0,q_*,v_0$ with
$S_0=\frac12(v_0k_0^\top+k_0v_0^\top)$ of eigenvalues $1$ and $-0.05$, and $500$ Adam steps on $k,q,v$ with the squared
loss. Softmax attention replaces $VXX^\top AX$ by $VX\operatorname{softmax}(X^\top AX)$. We report the medians,
interquartile ranges and eight individual runs of $\sigma_2(S_t)/\sigma_1(S_t)$ over $30$ runs.

\subsection{\Cref{fig:cnn}}

One-dimensional CNN with $L=2$, filters $w_1=(a,b)$ and $w_2=(c,e)$, strides one, $\sigma(z)=z^2$. Then
\[
f(x)=c\,(ax_1+bx_2)^2+e\,(ax_2+bx_3)^2=F_{11}x_1^2+2F_{12}x_1x_2+F_{22}x_2^2+2F_{23}x_2x_3+F_{33}x_3^2,
\]
so $\cH\cong\RR^5$. The neural model is the real zero set of
\begin{gather*}
F_{12}F_{23}-F_{11}F_{33},\qquad F_{12}^3-F_{11}F_{12}F_{22}+F_{11}^2F_{23},\\
F_{11}F_{23}^2+F_{12}^2F_{33}-F_{11}F_{22}F_{33},\qquad F_{23}^3-F_{22}F_{23}F_{33}+F_{12}F_{33}^2,
\end{gather*}
and its singular locus is the line $\RR x_2^2$. This line is computed by the two subnetworks $((0,b),(c,0))$ and
$((b,0),(0,c))$, and two branches of $\cF$ cross transversally along it. We take the linear slice
$\{F_{11}=F_{23},\ F_{12}=F_{33}\}$ with coordinates $(\xi,\eta,\zeta)=(F_{11},F_{12},F_{22})$. Then
\[
\cF\cap\{F_{11}=F_{23},\ F_{12}=F_{33}\}=\{\xi^3+\eta^3=\xi\eta\zeta\},
\]
the cone over the folium of Descartes, whose double line $\xi=\eta=0$ is the singular line $\RR x_2^2$. The figure uses
the coordinates $(\sqrt3\,\xi,\sqrt3\,\eta,\zeta)$, which are orthonormal for the Bombieri inner product, and shows the
ball of radius $1.6$ about the origin.

\end{document}